%% file: main.tex
\documentclass{article}

\usepackage{PRIMEarxiv}
\usepackage[utf8]{inputenc} 
\usepackage[T1]{fontenc}    
\usepackage{hyperref}       
\usepackage{url}            
\usepackage{booktabs}       
\usepackage{amsfonts}       
\usepackage{nicefrac}       
\usepackage{lipsum}
\usepackage{graphicx}
\graphicspath{{media/}}     
\usepackage[caption=false,font=footnotesize]{subfig}
\usepackage{comment}
\usepackage{tikz}
\usepackage{hyperref}
\usepackage{pifont}
\usepackage{pifont,xcolor}
\hypersetup{
    colorlinks=true,
    linkcolor=blue,
    filecolor=magenta,      
    urlcolor=cyan,
    citecolor=blue,
    }
\usepackage{url}
\usepackage{xcolor}
\usepackage{wrapfig}
\usepackage[caption=false,font=footnotesize]{subfig}
\makeatletter
\newcommand*{\rom}[1]{\expandafter\@slowromancap\romannumeral #1@}
\makeatother
\input{symbols.tex}
\definecolor{pleasant_red}{HTML}{FFB6C1}  
\definecolor{golden_color}{HTML}{FFD700}   
\definecolor{spring_green}{HTML}{00FF7F}   
\title{Physics-Inspired Binary Neural Networks:\\ Interpretable Compression with Theoretical Guarantees}

\author{
  Arian Eamaz$~^{*}$\\
  Department of Electrical and Computer Engineering \\
  University of Illinois Chicago \\
  Chicago, IL, USA \\
  \texttt{aeamaz2@uic.edu} \\
  \And Farhang Yeganegi \thanks{The first two authors contributed equally to this work.} \\
  Department of Electrical and Computer Engineering \\
  University of Illinois Chicago \\
  Chicago, IL, USA \\
  \texttt{fyegan2@uic.edu} \\
  \And Mojtaba Soltanalian \\
  Department of Electrical and Computer Engineering \\
  University of Illinois Chicago \\
  Chicago, IL, USA \\
  \texttt{msol@uic.edu} \\
}

\begin{document}
\maketitle

\begin{abstract}
Why rely on dense neural networks and then blindly sparsify them when prior knowledge about the problem structure is already available? Many inverse problems admit algorithm-unrolled networks that naturally encode physics and sparsity. In this work, we propose a Physics-Inspired Binary Neural Network (PIBiNN) that combines two key components: (i) data-driven one-bit quantization with a single global scale, and (ii) problem-driven sparsity predefined by physics and requiring no updates during training. This design yields compression rates below one bit per weight by exploiting structural zeros, while preserving essential operator geometry. Unlike ternary or pruning-based schemes, our approach avoids ad-hoc sparsification, reduces metadata overhead, and aligns directly with the underlying task. Experiments suggest that PIBiNN achieves advantages in both memory efficiency and generalization compared to competitive baselines such as ternary and channel-wise quantization.
\end{abstract}


\section{Introduction}
\label{sec1}
Large neural networks, including Large Language Models (LLMs), face challenges related to their size, such as requiring vast amounts of memory for storage and retrieval, consuming substantial power, and being inefficient in communicating over networks or storage for inference. This highlights the need for efficient compression methods \cite{ghane2024one,guo2018survey,daliri2024unlocking}. Quantization involves using fewer bits to store each model parameter, while pruning involves setting some parameter values to zero \cite{guo2018survey,frantar2023qmoe,liang2021pruning}. One effective way to significantly compress a network through quantization is by applying coarse quantization, or one-bit quantization, which converts parameters into binary values \cite{nagel2021white,zhang2022quantization}. A recent paper demonstrates that an LLM with just $1.58$-bit or ternary weights (i.e., $\{-1,0,1\}$) in its linear layers can perform as well as a high-resolution model with $3\times$ more layers  \cite{ma2024era}. 
\begin{figure}[t]
	\centering
		{\includegraphics[width=.98\columnwidth]{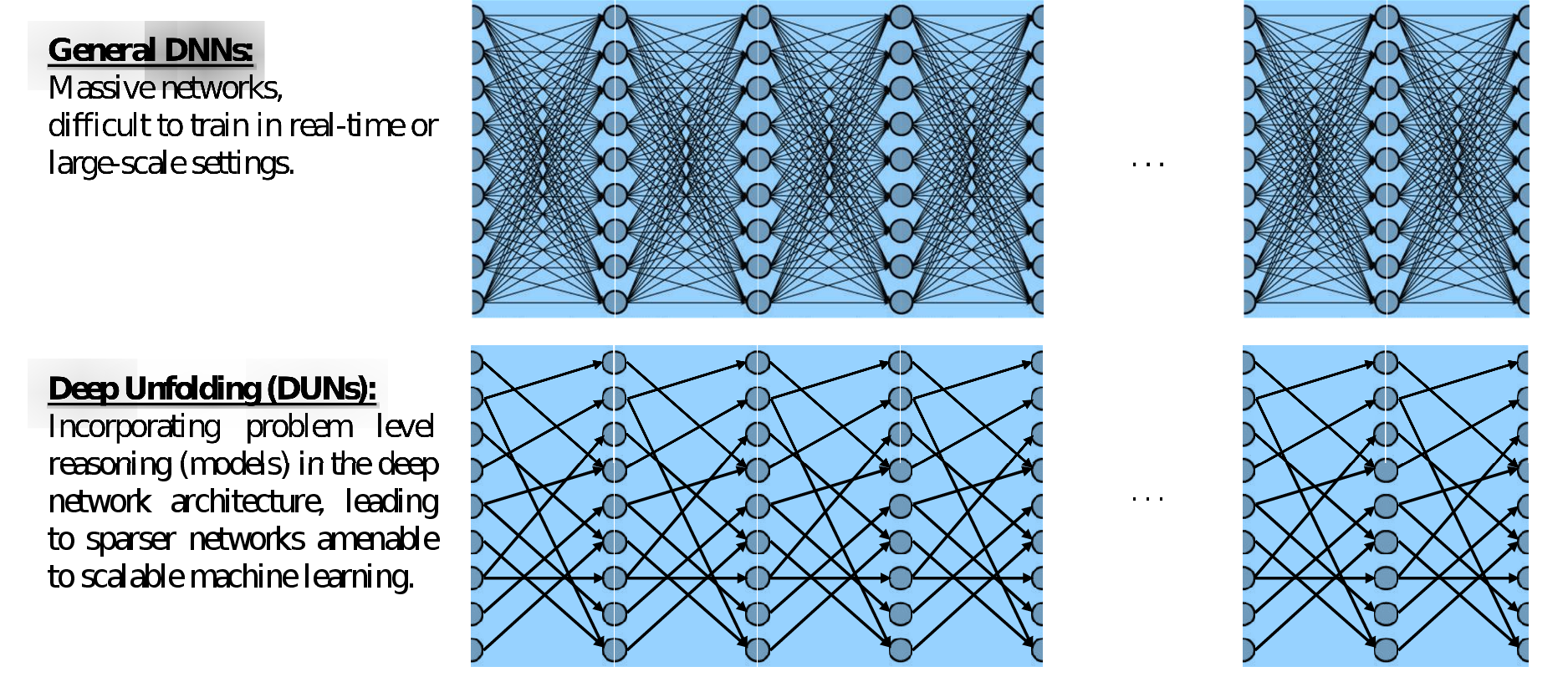}}
        \vspace{-.2cm}
	\caption{General DNNs vs DUNs. DUNs appear to be an excellent tool in large-scale  machine learning applications due to their inherent \emph{physics-inspired sparsity.}}
\vspace{-14pt}
\label{figure_0}
\end{figure}

Although pruning and quantization are popular sparsification techniques, in many tasks, we already have some form of a priori knowledge about the problem. Instead of blindly using a dense network and then compressing it with such methods, this knowledge can be embedded directly into the architecture. In fact, there exists a class of neural architectures in which the sparsity arises from the network topology and from the physics or mathematics of the underlying task, thereby avoiding manual zeroing of neurons \cite{zhang2023physics,le2015deep,gilton2021deep}. 

The canonical example is algorithm unrolling (also called deep unfolding): a classical iterative solver for the target problem is \textbf{unrolled} into a feed-forward network, with each iteration mapped to one layer and the forward pass emulating a fixed number of solver steps. Such Deep Unrolled Networks (DUNs) embed problem operators (e.g., sensing matrices, proximal maps, or linear transforms) and constraints directly into the architecture, yielding structured sparsity, far fewer trainable parameters, and strong sample efficiency compared with generic deep networks \cite{yang2018admm}. By leveraging domain knowledge, DUNs can solve optimization problems efficiently (see Fig.~\ref{figure_0}) and often train faster while generalizing better \cite{chen2022learning}. For instance, in sparse signal recovery, unrolled models have been shown to surpass classical solvers in accuracy at substantially lower inference cost.
 
For challenging problems, increasing the number of iterations is straightforward and computationally manageable in classical methods. However, in DUNs, adding iterations equates to adding layers and parameters, such as weight matrices, which can be inefficient for inference. Pruning the weights may not be a viable solution because each weight in a DUN plays a critical role in solving the primary objective, and removing them would be counterproductive. Therefore, a more practical approach to compressing the network is to employ weight quantization, especially for linear transformation weights. 
\vspace{-8pt}
\subsection{ Motivations and Contributions}
While one-bit quantization has shown great success in dense architectures like Transformers (e.g., BitNet proposed in \cite{wang2023bitnet}), applying it to DUNs presents unique challenges and opportunities. Our work is the first to explore this, demonstrating that the inherent sparsity from the underlying inverse problems generating the data allows for even greater compression than state-of-the-art methods and reduces training data requirements. One of the opportunities algorithm unrolling offers is interpretability: its algorithmic structure guides compression by assigning each weight and nonlinearity a clear task-level role rather than leaving them as opaque learned parameters. Moreover, this structure provides access to provable properties that naturally translate into theoretical discussions and explicit bounds on generalization ability. We refer to the proposed network as Physics-Inspired Binary Neural Network (\textbf{PIBiNN}).
Our main contributions in this paper are:\\
\textbf{1) Deepening a DUN is no longer a computational issue.}
We apply Quantization-Aware Training (QAT) to the unrolling algorithm, allowing the network to adapt to quantization noise during training \cite{nagel2021white}. While QAT typically requires costly pre-training \cite{nagel2020up}, this burden is greatly reduced for sparse architectures such as DUN, which need far less data and compute. Moreover, using one-bit weights makes it possible to efficiently increase the number of layers or iterations in the algorithm: our results show that a one-bit DUN with $4\times$ more layers surpasses the high-resolution model in both training and inference.\\
\textbf{2) Data-driven matrix-wise binarization across all layers.}
The choice of scale strongly affects quantization performance. Unlike prior work \cite{wang2023bitnet,ma2024era} that computes layer-wise scales in closed form by minimizing squared error between full-precision and binary weights, we adopt a data-driven matrix-wise approach. Our method learns a \emph{single} global scale shared across all layers, greatly reducing parameter overhead. Using one scale is also more efficient at inference, as it avoids per-layer rescaling and simplifies deployment. Interestingly, our numerical results show that, with careful design, this one-scale scheme achieves performance comparable to more complex channel-wise quantization.\\
\textbf{3) Embedding sparse structure from problem physics into the network.} By embedding the sparse structure of the problem directly into the network, we exploit both data and the problem’s inherent physics, making training more effective. For example, block-structured matrices naturally induce sparse, block-diagonal weights, so sparsity emerges from the unrolled architecture itself rather than from post-hoc pruning. This predefined sparsity reduces complexity, improves scalability, and requires no updates during training. Compared to other schemes, the PIBiNN learns a single scale and automatically obtains zeros from problem-driven sparsity, while ternary requires a third state ($0$) and per-block scales. This reduces metadata overhead and achieves fewer effective bits per link. To demonstrate this numerically, we compare both approaches on synthetic and real datasets, showing that sparsity arising from the structure can lead to better performance than sparsity applied during training epochs.\\
\textbf{4) Theoretical analysis and convergence guarantees of one-bit algorithm unrolling.}
Previous work \cite{chen2020understanding} analyzed the generalization ability of algorithm unrolling for quadratic programs, showing that the generalization gap depends on key algorithmic properties, such as the convergence rate. While the study examined features beyond simple gradient descent, it could not derive a convergence rate when a neural network was used for the update process. We extend this analysis to sparse quadratic programs, providing the convergence rate for our one-bit algorithm unrolling. A key distinction of our approach is that each algorithm iteration is transformed into a separate layer of the unrolled network, unlike \cite{chen2020understanding}, where the same network (with at least 10 layers) is reused for every update.

From a convergence perspective, prior work on Learned Iterative Shrinkage-Thresholding Algorithm (LISTA) \cite{chen2018theoretical} requires the weights of the DUN to belong to a specific ``good'' set, which is non-empty, to guarantee convergence. Ensuring such conditions for binary weights is challenging. We address this by modifying the proof framework to establish relaxed, achievable conditions for one-bit LISTA. Moreover, whereas previous analyses assume a fixed sensing matrix \cite{chen2018theoretical,chen2021hyperparameter,aberdam2021ada,yang2020learning}, our work also analyzes the robustness of the algorithm when the sensing matrix varies.\\
\textbf{5) Numerical ablation study for a large model.}
We conducted a detailed numerical analysis on a large-scale network where the data arises from an inverse problem with a block-sparse structure. Comparing a fully connected network to our PIBiNN, we show that the parameter count drops from 3 billion to just 100K one-bit parameters, achieving a $\mathbf{99.9\%}$ compression rate in bits. Despite this significant reduction, the model maintains competitive performance. 

\emph{Notation:} Throughout this paper, we use bold lowercase and bold uppercase letters for vectors and matrices, respectively. $\|\cdot\|$ denotes the second norm for both vector and matrix. The set $[n]$ is defined as $[n]=\left\{1,\cdots,n\right\}$. The covering number
$\mathcal{N}\left(\epsilon,\mathcal{F},\|\cdot\|_k\right)$ is defined as the cardinality of the smallest subset $\mathcal{F}^{\prime}$ of $\mathcal{F}$ for which every element of
$\mathcal{F}$ is within the $\epsilon$-neighborhood of some element of $\mathcal{F}^{\prime}$ with respect to the $k$-th norm $\|\cdot\|_k$. 
The hard-thresholding (HT) operator is defined as 
$\operatorname{HT}(x) = x \mathbb{I}_{x\geq 0}$,      
where $\mathbb{I}_{x\geq 0}$ is the indicator function. The soft-thresholding (ST) operator is represented by
$\operatorname{ST}_{\theta}(x) = \operatorname{sgn}(x) \left(|x|-\theta\right)$.
We define the operators $\operatorname{Ro}(\mbx_{1:u})$ and $\operatorname{Co}(\mbx_{1:u})$ as the row-wise and column-wise concatenations of the vectors $\{\mbx_i\}_{i=1}^u$, respectively. For a matrix $\mbB$, the operator $\operatorname{BlockDiag}_{u}(\mbB)$ constructs a block diagonal matrix by repeating the matrix $\mbB$, $u$ times along its diagonal entries. For a matrix $\mbQ\in\mathbb{R}^{n\times n}$ and an index set $\mathcal{S}$ of cardinality $s$, we define $\bar{\mbQ} = \mbQ[\mathcal{S}, :]$ as the $s \times n$ submatrix of $\mbQ$ consisting of the rows indexed by $\mathcal{S}$. Likewise, we define  
$\Tilde{\mbQ} = \mbQ[\mathcal{S}, \mathcal{S}]$ as the $s \times s$ submatrix formed by selecting both rows and columns indexed by $\mathcal{S}$. Finally, for a vector $\mbb$, we denote by $\bar{\mbb} = \mbb[\mathcal{S}]$ the $s$-dimensional subvector corresponding to the entries indexed by $\mathcal{S}$.
\vspace{-5pt}
\section{Algorithm Unrolling}
\label{sec2}
Assume we have the following optimization problem:
\begin{equation}
\label{eq1}
\mathcal{P}:~\underset{\mbx\in\mathcal{K}}{\textrm{minimize}}\quad f(\mbx)+\gamma \mathcal{R}(\mbx),
\end{equation}
where the target parameter $\mbx$ belongs to an arbitrary set $\mathcal{K}$ and the regularizer $\mathcal{R}(\cdot)$ is applied to $\mbx$
with a tuning parameter $\gamma$. 
If the objective is differentiable, one approach to solving the problem is to use the gradient descent method,
\begin{equation}
\label{eq2}
\mbx_{k+1}=\mbx_{k}-\alpha \nabla f(\mbx_{k})-\alpha\gamma \nabla \mathcal{R}(\mbx_{k}).
\end{equation}
The core idea of algorithm unrolling is to substitute the gradient term, specifically the gradient of the regularizer, with a neural network. In the proximal formulation, this substitution can be expressed as an operator $\mathcal{N}_{\btheta}$, where the weights $\btheta$ are learned from the training data:
\begin{equation}
\label{eq3}
\mbx_{k}=\mathcal{N}_{\btheta}\left(\mbx_{k-1}-\alpha \nabla f(\mbx_{k-1})\right).
\end{equation}
More generally, this problem can be addressed using any applicable algorithm that can be unrolled into a DUN:
\begin{equation}
\mbx_{k}=\operatorname{Alg}^{k}_{\btheta_k}\left(\mbx_{k-1},\mbb\right),~k\in[K],
\end{equation}
with the algorithm parameter set $\btheta_k$ 
and $\mbb$ represents the measurements.
After processing $K$ iterations, the learning problem is then to find the best model by minimizing the empirical loss function
\begin{equation}
\underset{\btheta}{\textrm{minimize}} \quad \mathbb{P}_N\ell_{\btheta},
\end{equation}
where $\btheta=\left\{\btheta_k\right\}^{K}_{k=1}$ are learnable parameters and 
$\mathbb{P}_{N}\ell_{\btheta}=\frac{1}{N}\sum^{N}_{i=1}\left\|\operatorname{Alg}^{K}_{\btheta}\left(\mbx_{K-1},\mbb_i\right)-\mbx^{\mathrm{opt}}_i\right\|$,
with $\mbx^{\mathrm{opt}}$ being the true labels. Using a neural network (typically with more than $10$ layers) for each iteration can lead to training instability due to the large parameter space and challenges in providing theoretical guarantees \cite{gilton2021deep}. Moreover, a network used as a proximal operator may struggle to accurately replicate the behavior of the actual operator for structured parameters, such as sparse signals or low-rank matrices.

Another approach is to design a network based on the algorithm's iterative process, where the number of layers in the network corresponds to the number of algorithmic iterations. In this framework, the activation function within each layer mimics the proximal operator rather than utilizing a multi-layer network at each iteration. One well-known example is the Ccompressed Ssensing (CS) problem, for which ISTA is a classic solver \cite{daubechies2004iterative}.
To unroll CS solvers, numerous DUNs have been proposed, among which the most widely known is LISTA:
\begin{equation}
\label{eq12}
\mathbf{x}_{k}=\operatorname{ST}_{\theta_{k}}\bigl(\mathbf{x}_{k-1}-\mathbf{W}_k^{\top}\bigl(\mathbf{A} \mathbf{x}_{k-1}-\mathbf{y}\bigr)\bigr),~k\in[K],
\end{equation}
where $\bigl\{\theta_{k},\mbW_{k}\bigr\}^{K}_{k=1}$ are the parameters to be trained, $\mbA$ is a sensing matrix, and $\mby$ is the measurement vector.
This model plays a central role in various applications, ranging from image denoising to ultrasonic image recovery \cite{chen2018theoretical,zhang2023physics,monga2021algorithm}.
\section{One-bit Unrolling: Training Process and Solvers}
\label{sec3}

Denote the weights or linear transformations within each layer by $\mbW^{(k)}=[W_{i,j}^{(k)}]\in\mathbb{R}^{n_1\times n_2}$ for $k\in[K]$. Our approach in QAT consists of two stages, discussed below.

\textbf{Stage I}: In this stage, we consider an scaled one-bit quantization of weights. 

Specifically, we investigate two approaches: (i) training with the straight-through gradient method or lazy projection, and (ii) conducting QAT as a regularized optimization process. The straight-through gradient method is equivalent to a projected gradient method, as follows:
\begin{equation}
\label{eq19}
\left\{\begin{array}{l}
\mbz^{(t)}=\lambda_{0}\operatorname{sign}\bigl(\btheta^{(t)}\bigr),\\
\btheta^{(t+1)}=\btheta^{(t)}-\eta_t \nabla \mathbb{P}_N\ell_{\btheta}|_{\btheta=\mbz^{(t)}},
\end{array}\right.
\end{equation}
where 
$\btheta^{(t)}$ represents the network's parameters at the $t$-th epoch.
Notably, in this work, we apply the quantization operator exclusively to the network weights $\mbW^{(k)}$ as
$\lambda_0\operatorname{sign}\bigl(W_{i,j}^{(k)}\bigr)= R_{i,j}^{(k)}$.
To help the solver escape local optima that may arise due to hard projections and to encourage progress toward better solutions, the effect of hard projections can be alleviated through regularization \cite{bai2018proxquant}:
\begin{equation}
\label{eq20}
\mathcal{P}^{(\text{QAT})}:~\underset{\btheta}{\operatorname{minimize}}~\mathbb{P}_N\ell_{\btheta}
+\beta \mathcal{R}\left(\btheta\right).
\end{equation}
For example, one may use $\ell_1$ regularization
$\mathcal{R}\left(\btheta\right) =\sum_{j} \min\left\{\left|\theta_{j}-\lambda_0\right|, \left|\theta_{j}+\lambda_0\right|\right\}$.
In this version, the update process in each epoch becomes
\begin{equation}
\label{eq22} 
\btheta^{(t+1)}=\operatorname{prox}_{\eta_t \beta \mathcal{R}}\bigl(\btheta^{(t)}-\eta_t\nabla \mathbb{P}_N\ell_{\btheta}|_{\btheta=\btheta^{(t)}}\bigr).
\end{equation}

\textbf{Stage II:} Previous research used $\frac{\sum_{i,j}|W_{i,j}|}{n_1n_2}$ as the per-layer scale for one-bit quantization \cite{wang2023bitnet}. However, we propose a data-driven method to determine this value more effectively. Additionally, we use a single scale for all layers of the network, which can introduce more efficiency to the trained model during the inference phase.

Let $\widehat{\btheta}$ represent the solution obtained after \textbf{Stage I}, where the weights of the network are given by 
$\widehat{W}_{i,j}^{(k)}=\widehat{R}_{i,j}^{(k)},~\widehat{R}_{i,j}^{(k)}\in\{-\lambda_0,\lambda_0\},(i,j)\in[n_1]\times[n_2],~ k\in[K]$
or more compactly as $\widehat{\mbW}^{(k)}=\widehat{\mbR}^{(k)}$ for $k\in[K]$. In this stage, we change $\lambda_0$ to $\lambda_0\lambda$, where the goal is to optimize the following objective with respect to $\lambda$ given the training dataset:
\begin{equation}
\label{eq26}
\mathcal{P}^{(\lambda)}:~\underset{\lambda}{\operatorname{minimize}}~\mathbb{P}_N\ell_{\lambda},
\end{equation}
where 
$\mathbb{P}_N\ell_{\lambda}=\frac{1}{N}\sum_{i=1}^{N}\left\|\operatorname{Alg}^{K}_{\lambda \widehat{\btheta}}\bigl(\mbx_{K-1},\mbb_i\bigr)-\mbx^{\mathrm{opt}}_i\right\|$
Note that $\lambda$ only affects the weights of the network and not the parameters of the activation function within $\lambda\widehat{\btheta}$.
\section{Towards Sparse Network with Algorithm Unrolling}
\label{sec4}
Although algorithm unrolling reduces the number of parameters compared to a Fully Connected Network (FCN), it still preserves full connectivity between the input and output at each layer. \emph{Note that even a DUN only emulates the solver without reflecting the structure of the sensing matrix, effectively remaining a dense network. The core idea of the PIBiNN is to impose the problem’s structure onto the dense DUN.} If the original inverse problem involves a sparse, block-structured matrix, we can partition the inputs, where each partition of the output is generated by only a corresponding subset of the input rather than all input elements. This structure results in a sparse, block-structured matrix that the iterative algorithms must respect, reflecting the original problem's sparsity pattern.



For example, consider the following linear matrix equation as the original inverse problem, where the goal is to recover 
$\mbX=\operatorname{Ro}(\mbx_{1:u})\in\mathbb{R}^{p\times u}$
from the measurement matrix 
$\mbY=\operatorname{Ro}(\mby_{1:u})\in\mathbb{R}^{v\times u}$,
using the sensing matrix 
$\mbA\in\mathbb{R}^{v \times p}$,
$\mbY = \mbA\mbX$.
This equation can be reformulated as 
$\operatorname{vec}\left(\mbY\right) = \mbA^{\prime}\operatorname{vec}\left(\mbX\right)$, where $\mbA^{\prime}=\left(\mbI_{u}\otimes \mbA\right)$.
Therefore, we have a block-structured matrix whose diagonal parts are $\mbA$. In this equation, each column of $\mbX$ is assigned to the corresponding column in the measurement matrix $\mbY$, and the column vectors of the input matrix do not influence the output of other columns. To solve this set of linear equations, we can either apply proximal solver to each equation individually or treat the concatenation of all equations as a single update process, as described in the following model for $k\in[K]$:
\vspace{-5pt}
\begin{equation}
\begin{aligned}
\operatorname{Co}(\mbx_{1:u,k})=\mathcal{H}_{\theta_k}\bigl(\operatorname{Co}(\mbx_{1:u,k-1})-\widehat{\mbW}_k^{\top}\left(\mbA^{\prime}\operatorname{Co}(\mbx_{1:u,k-1})-\operatorname{Co}(\mby_{1:u})\right)\bigr).
\end{aligned}
\end{equation} 
where $\mathcal{H}_{\theta}$ is a proximal operator. Since each input partition is independent of the others, and each corresponding output partition is influenced solely by its respective input, we can construct the following weight matrix that is sparse and blocky, mirroring the structure of the sensing matrix:
$\widehat{\mbW}_k=\operatorname{BlockDiag}_{u}(\mbW_k)\in\{-\lambda,0,\lambda\}^{vu\times pu},~k\in[K]$, where $\mbW_k\in\{-\lambda,\lambda\}^{v\times p}$ for $k\in[K]$ are learnable parameters. This block structure ensures that each part of the model interacts only with the relevant input, maintaining the sparsity of the system.
In some problems, the sensing matrix $\mbA$ itself is sparse. In such cases, even certain elements within the partitions are independent of one another, generating their outputs based solely on these corresponding elements. 
\vspace{-5pt}
\section{Generalization Ability}
\label{sec5}
Although the scheme proposed in this paper can be applied to a variety of problems using algorithm unrolling solvers, the theoretical guarantees are established for a hybrid architecture where the neural module is an energy function of the form $E_{\phi}((\bxi,\mbb),\mbx)=\frac{1}{2}\mbx^{\top}\mbQ_{\phi}(\bxi)\mbx-\mbb^{\top}\mbx$, with $\mbQ_{\phi}$ a varying sensing matrix or a neural network that maps $\bxi$ to a matrix, and $\mbx\in\mathcal{K}_{\mathcal{S}}$ where $\mathcal{K}_{\mathcal{S}}$ is defined as $\mathcal{K}_{\mathcal{S}}\triangleq\{\mbx=[x_i]\in\mathbb{R}^n:x_{i}=0, \forall i\in[n]\setminus\mathcal{S},|\mathcal{S}|=s\}.$ This problem is equivalent to the sparse quadratic program. Note that we consider the set $\mathcal{S}$ fixed but unknown in our settings. Extending such guarantees to other problems would require different DUN architectures and, consequently, distinct theoretical analyses, which lie beyond the scope of this paper.

Each energy $E_{\phi}$ can be uniquely represented by $(\mbQ_{\phi}(\bxi),\mbb)$, so we can write the overall optimization problem as 
\begin{equation}
\label{a1}
\begin{aligned}
\mathcal{C}:\quad\underset{\mbx\in\mathcal{K}_{\mathcal{S}}}{\textrm{minimize}} \quad E_{\phi}((\bxi,\mbb),\mbx).
\end{aligned}
\end{equation}
Given $N$ i.i.d. training samples $\{\mbx^{\mathrm{opt}}(\mbQ^{\mathrm{opt}}(\bxi_i),\mbb_i)\}$, the learning task is then to find the best model by performing the optimization 
$\left(\underset{\btheta,\phi}{\textrm{minimize}} \quad \mathbb{P}_N\ell_{\btheta,\phi}\right)$, 
where 
\begin{equation}
\label{cost_1}
\mathbb{P}_{N}\ell_{\btheta,\phi}=\frac{1}{N}\sum^{N}_{i=1}\left\|\operatorname{Alg}^{K}_{\btheta}\left(\mbQ_{\phi}(\bxi_i),\mbb_i\right)-\mbx^{\mathrm{opt}}(\mbQ^{\mathrm{opt}}(\bxi_i),\mbb_i)\right\|.
\end{equation}
The algorithm we consider here is the following unrolled iterative procedure for $\delta \in (0,1]$ (for notational simplicity, we replace $\mbQ_{\phi}(\bxi)$ with $\mbQ$):
\begin{equation}
\label{update_sqp_1}
\mbx_{k} = \mathcal{H}_{\theta_k}\left(\delta\mbx_{k-1}-\mbW^{\top}_{k}\left(\bar{\mbQ}\mbx_{k-1}-\bar{\mbb}\right)\right),~k\in[K],   
\end{equation}
where $\mathcal{H}_{\theta}(\cdot)$ can be either ST or HT operator, and $\left\{\theta_k,\mbW_{k}\in\{-\lambda,\lambda\}^{s\times n}\right\}^{K}_{k=1}$ are trainable parameters. 

To present the generalization result of a given algorithm, we first define three important algorithmic properties as

    
    \textbf{Convergence:} The convergence rate of an algorithm indicates how quickly the optimization error diminishes as $k$ increases,
    \begin{equation}
    \begin{aligned}
    \label{cvg_eq}\left\|\operatorname{Alg}^{k}_{\btheta}\left(\bar{\mbQ},\bar{\mbb}\right)-\mbx^{\mathrm{opt}}\left(\mbQ,\mbb\right)\right\|\leq\operatorname{Cvg}(k)\left\|\mbx_0-\mbx^{\mathrm{opt}}\left(\mbQ,\mbb\right)\right\|.
    \end{aligned}
    \end{equation}
    \textbf{Stability:} The robustness of an algorithm to small perturbations in the optimization
    objective, which corresponds to the perturbation of $\mbQ$ and $\mbb$ in the quadratic case is demonstrated as,
    \begin{equation}
    \begin{aligned}
    \label{stab_eq}\left\|\operatorname{Alg}^{k}_{\btheta}\left(\bar\mbQ,\bar\mbb\right)-\operatorname{Alg}^{k}_{\btheta}\left(\bar\mbQ^{\prime},\bar\mbb^{\prime}\right)\right\|\leq\operatorname{Stab}^{\bar\mbQ}(k)\left\|\Tilde\mbQ-\Tilde\mbQ^{\prime}\right\|+\operatorname{Stab}^{\bar\mbb}(k)\left\|\bar\mbb-\bar\mbb^{\prime}\right\|.
    \end{aligned}
    \end{equation}
    \textbf{Sensitivity:} The robustness to small perturbations in the algorithm parameters $\btheta$ is given by,
    \begin{equation}
    \label{sens_eq}\left\|\operatorname{Alg}^{k}_{\btheta}\left(\bar\mbQ,\bar\mbb\right)- \operatorname{Alg}^{k}_{\btheta^{\prime}}\left(\bar\mbQ,\bar\mbb\right)\right\|\leq \operatorname{Sens}^{\btheta}(k)\left\|\btheta-\btheta^{\prime}\right\|.
    \end{equation}
As we will show in Theorem~\ref{theorem_1}, these algorithmic properties play a critical role in determining the generalization ability of a given algorithm.
As demonstrated in \cite{chen2018theoretical}, to guarantee the convergence of the ST algorithm (e.g., LISTA) for the sparse quadratic program, the network weights must belong to a specific set defined as follows:
\begin{definition}
\label{def_1}
Given $\mbQ\in\mathbb{R}^{n\times n}$ and $\mathcal{I}=\{-\lambda,\lambda\}^{s\times n}$, a weight matrix is ``good'' if it belongs to $\mathcal{X}_{\mbW}(\mbQ)$  defined as
\begin{equation}
\label{cvg1}
\begin{aligned}
\arg\min_{\mbW\in\mathcal{I}}\left\{\max|W_{i,j}|:\mbW_i^{\top}\bar\mbQ_i=1,\max_{i\neq j}|\mbW_{i}^{\top}\bar\mbQ_j|=\mu_Q\right\}.
\end{aligned}
\end{equation}
\end{definition}
Selecting the weight matrices from the set of ``good'' weights imposes a significant constraint within the class of binary weights. This restriction suggests that satisfying the first term of \eqref{cvg1} with a binary weight may be particularly difficult. Even under the assumption that the set in Definition~\ref{def_1} is non-empty, identifying the minimizer of \eqref{cvg1} remains a highly challenging task. It is also important to note that the set of ``good'' weights in Definition~\ref{def_1} only guarantees convergence. However, as we will see in Theorem~\ref{theorem_1}, stability and sensitivity are also key factors, alongside convergence, in determining the generalization ability of an algorithm. Therefore, our objective is to develop an alternative strategy that not only ensures uniform convergence across all training samples but also guarantees generalization for binary models.

In this paper, our strategy centers on analyzing the term 
$\left\|\delta\mbI - \mbW_{\mathcal{S},k}^{\top} \Tilde\mbQ\right\|$
for all $k\in[K]$, where $\mbW_{\mathcal{S},k}$ denotes the submatrix of $\mbW_k$ formed by the columns indexed by $\mathcal{S}$. As shown in Appendices~\ref{App_B}, \ref{App_C}, \ref{App_D}, and \ref{App_E}, this term plays a pivotal role in determining the algorithm properties. Our analysis reveals that this spectral norm must have a certain bound for all $k\in[K]$ to ensure a guaranteed generalizability. 

$\bullet$ \textbf{ST algorithm:} The algorithmic properties of the ST update process can be summarized in the following lemma:
\begin{table}[t]
\caption{Algorithm properties of the ST update process in \eqref{update_sqp_1}.
}
\centering
\resizebox{0.6\columnwidth}{!}{
\begin{tabular}{ c || c|| c|| c|| c
}
\hline
$\operatorname{Cvg}_\theta(k)$ &  $\operatorname{Stab}_\theta^{\bar\mbQ}(k)$ &  $\operatorname{Stab}_\theta^{\bar\mbb}(k)$&  $\operatorname{Sens}_\theta^{\mbW_i}(k)$&  $\operatorname{Sens}_\theta^{\theta_i}(k)$\\[0.5 ex]
\hline
$\mathcal{O}\left(\alpha_{\theta}^{k}\right)$ &  $\mathcal{O}\left(1-\alpha^{k}_{\theta}\right) $ &  $\mathcal{O}\left(1-\alpha^{k}_{\theta}\right)$&  $\mathcal{O}\left(\alpha^{k-i}_{\theta}\right)$&  $\mathcal{O}(\alpha^{k-i}_{\theta})$ 
\\[0.5 ex]
\hline
\end{tabular}
}
\label{table_3}
\vspace{-10pt}
\end{table}
\begin{lemma}
\label{lem_1}
The convergence rate, stability, and sensitivity of the ST algorithm in \eqref{update_sqp_1} are summarized in Table~\ref{table_3}, where 
\begin{equation}
\alpha_\theta=\sup_{k\in[K]}\sup_{\mbQ}\left(\left\|\delta\mbI - \mbW_{\mathcal{S},k}^{\top} \Tilde\mbQ\right\|+\mu_Q^{W_k} s\right).
\end{equation}
\end{lemma}
In Appendices~\ref{App_B}, \ref{App_D1}, and \ref{App_E1}, we present a thorough analysis and detailed proof of Lemma~\ref{lem_1}. As shown in Table~\ref{table_3}, if $\alpha_\theta\in(0,1)$, the ST update process exhibits bounded algorithmic properties and, consequently, guarantees generalization.

Unlike the set of ``good'' weights in Definition~\ref{def_1}, it is relatively straightforward to demonstrate that a binary network weight exists such that $\alpha_{\theta}\in(0,1)$. To illustrate this, based on our quantization scheme in Section~\ref{sec3}, we can express all weights as $\mbW_k=\lambda\mbB_k$, where $\mbB_k\in\{-1,1\}^{s\times n}$ for all $k\in[K]$. Then, we can write
$\left\|\delta\mbI-\mbW_{\mathcal{S},k}^{\top}\Tilde\mbQ\right\|\leq\delta+\lambda\|\mbB_{\mathcal{S},k}\|
\|\Tilde\mbQ\|$.
Define 
$h=\sup_{k\in[K]}\sup_{\mbQ}\|\mbB_{\mathcal{S},k}\|\|\Tilde\mbQ\|$ 
and 
$\mu_\theta=\sup_{k\in[K]}\sup_{\mbQ}\mu_Q^{W_k}$.
By choosing $\delta$ appropriately and setting $\lambda = \mathcal{O}(h^{-1})$, we can guarantee that $\delta+\lambda h+\mu_\theta s<1$ for sufficiently small value of sparsity level $s$. It is important to note that this choice of $\delta$ and $\lambda$ is specifically designed for the worst-case scenario. In Appendix~\ref{App_F4}, we present an ablation study on the impact of $\delta$ on the convergence and generalization of the ST update process.

$\bullet$ \textbf{HT algorithm:} 
In Appendices~\ref{App_C}, \ref{App_D2}, and \ref{App_E2}, we show that the parameter 
$\alpha_\theta=\sup_{k\in[K]}\sup_{\mbQ}\|\delta\mbI-\mbW^{\top}_{\mathcal{S},k}\Tilde\mbQ\|$ 
plays a crucial role in all algorithmic properties of the HT update process. In terms of order complexity, the impact of $\alpha_\theta$ on all algorithm properties in this case is equivalent to Table~\ref{table_3}. Thus, $\alpha_\theta \in (0,1)$ provides a sufficient condition for the generalization of the HT update process. In Appendix~\ref{App_F4}, we present an ablation study on the impact of $\delta$ on the convergence and generalization of the HT update process.


We will now demonstrate how these algorithm properties affect the generalization gap of algorithm unrolling. It is generally known \cite{chen2020understanding,bartlett2005local} that the standard Rademacher complexity provides global estimates of the complexity of the model space, which ignores the fact that the training process will likely pick models with small errors. Therefore, similar to \cite{chen2020understanding}, we resort to utilizing the local Rademacher complexity, which is defined as follows:
\begin{definition}
\label{local_rad_com}
The local Rademacher complexity of $\ell_{\mathcal{F}}$ at level $r$ is defined as $\mathbb{E} R_N \ell_{\mathcal{F}}^{\text {loc }}(r)$, where 
\begin{equation}
\vspace{-3pt}
\ell_{\mathcal{F}}^{\text {loc }}(r)\triangleq\left\{\ell_{\btheta,\phi}: \phi \in \Phi, \btheta \in \Theta, \mathbb{P}_N\ell_{\btheta,\phi}^2 \leq r\right\},
\vspace{-3pt}
\end{equation}
and $R_N \ell_{\mathcal{F}}^{\mathrm{loc}}(r)$ denotes the empirical local Rademacher complexity.
\vspace{-10pt}
\end{definition}
Based on Definition~\ref{local_rad_com} and our assumptions (see Appendix~\ref{assum}), the following theorem provides the generalization ability of an algorithm unrolling:
\begin{theorem}
\label{theorem_1}
The generalization ability of an algorithm unrolling utilized to solve the program $\mathcal{C}$ in \eqref{a1} for any $t>0$ is obtained as 
\begin{equation}
\begin{aligned}
\mathbb{E} R_n \ell_{\mathcal{F}}^{\text {loc }}(r)\leq\operatorname{Sens}(k)B_{\theta}+\widehat{\operatorname{Stab}}(k)\left(\sqrt{(\operatorname{Cvg}(k)+\sqrt{r})^2 C_1(N)+C_2(N, t)}+C_3(N, t)+4\right),
\end{aligned}
\end{equation}
where 
$\widehat{\operatorname{Stab}}(k)=\sqrt{2} s N^{-\frac{1}{2}}\operatorname{Stab}^{\bar\mbQ}(k)$
and 
$\operatorname{Cvg}(k)$
are worst-case stability and convergence, 
$B_{\btheta}=\frac{1}{2}\sup_{\btheta, \btheta^{\prime} \in \Theta}\left\|\btheta-\btheta^{\prime}\right\|$,
$C_1(N)=\mathcal{O}(\log \Gamma(N))$, $C_3(N, t)=\mathcal{O}\left(\frac{\log \Gamma(N)}{\sqrt{N}}+\frac{\sqrt{\log \Gamma(N)}}{e^t}\right)$, $C_2(N, t)=\mathcal{O}\left(\frac{t \log \Gamma(N)}{N}+\left(C_3(N, t)+1\right) \frac{\log \Gamma(N)}{\sqrt{N}}\right)$,
$\Gamma(N)=\mathcal{N}\left(\frac{1}{\sqrt{N}}, \ell_{\mathcal{Q}}, L_{\infty}\right)$.
\end{theorem}
Appendix~\ref{App_E} shows that the algorithm's sensitivity is directly tied to its convergence rate. The proof of Theorem~\ref{theorem_1} can be found in Appendix~\ref{app_rad}. 
\vspace{-5pt}



\section{Numerical Illustrations}
\label{sec6}
\vspace{-5pt}
In this section, we evaluate the effectiveness of our proposed physics-inspired binarization.\\ 
$\bullet$ \textbf{Synthetic dataset:} We set $m=50$ and $n=100$. The entries of the matrix $\mbA=[A_{i,j}]\in\mathbb{R}^{m\times n}$ are sampled i.i.d. from 
$A_{i,j}\sim\mathcal{N}(0,1/m)$.
To generate sparse vectors $\mbx^{\mathrm{opt}}$, we consider each of its entries to be non-zero following the Bernoulli distribution with $p=0.05$. The values of the non-zero entries are sampled from the standard Gaussian distribution, $\mathcal{N}(0,1)$. We generate a training set of $4000$ samples and a test set of $1000$ samples. 
During the first stage of training with the forward model in \eqref{eq12} using high-resolution weights, the learning rate is kept constant at $\eta_t=1e-03$ across all epochs $t$. In the \textbf{Stage I} of our proposed one-bit algorithm unrolling scheme, we initialize the learning rate at $\eta_0=1e-03$ and adjust it after every $10$ epochs with a decay factor of $0.9$. In \textbf{Stage II}, the learning rate remains fixed at $\eta_t=1e-03$ for all epochs.

Define the relative reconstruction error as 
$\mathrm{NMSE}\triangleq\frac{1}{N} \sum^{N}_{i=1}\frac{\left\|\mathbf{x}_{K}\left(\widehat{\btheta}, \mathbf{y}_i\right)-\mbx^{\mathrm{opt}}_i\right\|^2}{\left\|\mbx^{\mathrm{opt}}_i\right\|^2}$, 
where $\widehat{\btheta}$ represents the estimated parameters of the network. 
\begin{figure}[t]
	\centering
		{\includegraphics[width=0.4\columnwidth]{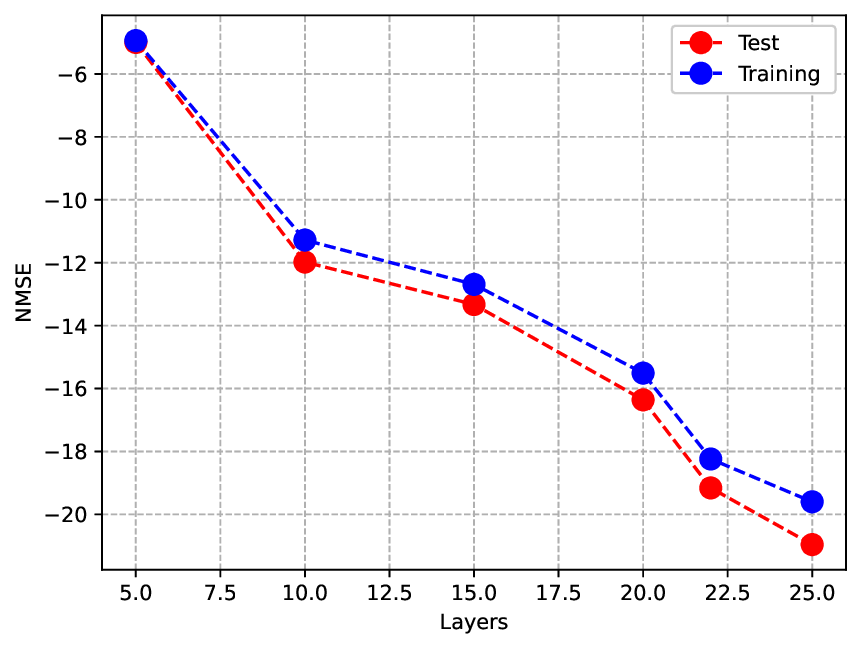}}
        \vspace{-.2cm}
	\caption{The impact of the number of layers on ``Train'' and ``Test'' NMSE in one-bit DUN.
    }
\label{figure_1}
\vspace{-13pt}
\end{figure}
Fig.~\ref{figure_1}(a) displays both the training and test results for one-bit DUN across different layer counts, specifically $K\in\{5, 10, 15, 20, 22, 25\}$. The results show that as the number of layers increases, the training and test NMSE (dB) improves. Table~\ref{table_1} provides the exact training and test errors corresponding to Fig.~\ref{figure_1}(a), offering a more detailed view of our findings. Additionally, Table~\ref{table_1} compares the training and test NMSE across three configurations: (i) a $5$-layer FCN with ReLU and ST activation functions, (ii) a $5$-layer DUN, and (iii) one-bit DUN with varying layer counts. Since the FCN does not leverage domain-specific knowledge from the model, we design $K$ layers with the configurations $\mbW_{1}\in\mathbb{R}^{m \times n}$ and $\{\mbW_{k}\in \mathbb{R}^{n\times n}\}_{k=2}^{K}$, leading to a higher bit count compared to the DUN with the same number of layers. The bit count for each model listed in Table~\ref{table_1} is provided in Appendix~\ref{App_F}.

\begin{table}[ht]
\caption{Comparison of training and test NMSE for three configurations: (i) a 5-layer FCN with ReLU and ST activation functions, (ii) a 5-layer DUN, and (iii) one-bit DUN with varying layer numbers.}
\centering
\resizebox{0.7\columnwidth}{!}{
\begin{tabular}{  c || c | c | c }
\hline
& Training NMSE (dB) & Test NMSE (dB) & $\#$ Bits \\[0.5 ex]
\hline
FCN with ReLU & $0.98$ & $2.34$ & $1440000$ \\
FCN with ST & $-4.44$ & $-1.34$ & $1440160$ \\
DUN with $5$ layers & $-19.20$ & $-16.40$ & $800160$ \\
One-Bit DUN with $5$ layers & $-5.00$ & $-4.94$ & $25160$ \\
One-Bit DUN with $10$ layers & $-11.98$ & $-11.28$ & $50320$ \\
One-Bit DUN with $15$ layers & $-13.33$ & $-12.69$ & $75480$ \\
One-Bit DUN with $20$ layers & $-19.04$ & $-17.42$ & $100640$ \\
One-Bit DUN with $22$ layers & $-19.33$ & $-18.24$ & $110704$ \\
One-Bit DUN with $25$ layers & $-20.96$ & $-19.30$ & $125800$ \\
\hline
\end{tabular}
}
\label{table_1}
\end{table}

As observed, the $5$-layer DUN outperforms the $5$-layer FCN with both ReLU and ST activation functions while requiring fewer stored bits. An intriguing observation is that, in training, we used $4000$ samples for the FCN, while for the DUN, we reduced the training set so that both models achieved the same level of test accuracy. Our numerical experiments show that the DUN requires approximately 10 times fewer training samples (about 400 samples) to achieve the same test accuracy as the FCN. This is a particularly significant finding, especially when dealing with very large datasets, where the cost of acquiring and processing training data can be substantial.
Furthermore, the $5$-layer FCN with an ST activation function achieves better training and test NMSE compared to its ReLU-based counterpart. This is likely due to the sparse nature of $\mbx^{\mathrm{opt}}$, where the ST operator can better mimic the sparsity behavior compared to the ReLU function. However, incorporating threshold parameters in ST slightly increases the bit count required to store the FCN model's parameters compared to the FCN with the ReLU activation function. For more details, see Table~\ref{table_1} and the bit count analysis provided in Appendix~\ref{App_F}. By comparing the one-bit DUN with the DUN in Table~\ref{table_1}, it is evident that when the one-bit DUN reaches $20$ layers, its test performance surpasses that of the DUN with $5$ layers. For example, the one-bit DUN with $20$ layers outperforms the $5$-layer DUN in test NMSE while achieving an $86\%$ compression ratio. This highlights the efficiency of the one-bit DUN architecture for deeper networks. While loss is generally expected to decrease monotonically with each layer in an unrolled algorithm, inconsistencies can appear in the convergence. Despite this issue, our one-bit DUN approximates monotonic behavior more closely than the high-resolution network (see Appendix~\ref{App_F2} for further details). The computational time required to run the different stages of the proposed one-bit DUN across various layer configurations is reported in Appendix~\ref{App_F3}. The numerical investigation of the impact of the scale design process is presented in Appendix~\ref{App_F40}, covering various fixed $\lambda$ values alongside the optimal value.

$\bullet$ \textbf{Large-scale dataset:} We now move beyond the one-bit DUN itself and, by leveraging a priori knowledge of problem structures, numerically scrutinize the performance of the proposed PIBiNN, which achieves further sparsification (as discussed in Section~\ref{sec4}). We generate sub-matrices $\mbA\in\mathbb{R}^{50\times 100}$ using the same settings as before and construct the block matrix $\mbA^{\prime}$ using $100$ such sub-matrices (see the sparse structure settings in Section~\ref{sec4}). In the PIBiNN, the weight matrices follow the sparse structure of the sensing matrices. We report the training and test results for networks with $10$ and $20$ layers, including the number of parameters for both the FCN and PIBiNN, as summarized in Table~\ref{table_4}(left panel). \emph{By leveraging the domain knowledge of the system, the number of parameters in the PIBiNN is drastically reduced from $3$ billion in the FCN to just $100$ thousand, representing a significant reduction in model complexity.} This yields a $99\%$ compression rate in both parameter reduction and bit ratio compared to the dense network. The training and test results demonstrate the successful performance of PIBiNN, even for large networks.

We now assess the individual contributions of the two core components of PIBiNN: the incorporation of \textbf{problem-driven sparsity} and \textbf{one-bit quantization}. Using the large-scale dataset setting with $20$ layers, we evaluate how each component affects performance and generalization. Note that in this comparison, we evaluate the effect of each component only on the dense DUN. As shown in Table~\ref{table_16} (right panel), applying only the problem’s sparsity without quantization slightly degrades training performance but improves test accuracy, indicating stronger generalization. An interesting observation from this experiment is that, despite a nearly $99\%$ reduction in the number of parameters, the training performance experiences only a slight drop, while test accuracy and generalization gap improve. \emph{When both sparsity and quantization are applied, the generalization performance improves even further, highlighting the complementary benefits of combining structure with extreme compression.}

\begin{table*}[ht]
\centering
\caption{(Left) Illustration of the impact of incorporating the physics of the problem into learning on the number of parameters, as well as the training and test results for the PIBiN. (Right) Effect of structured sparsity and one-bit quantization on DUN accuracy for large networks.}
\begin{minipage}{0.5\textwidth}
\centering
\resizebox{\columnwidth}{!}{
\begin{tabular}{ c | c | c| c | c }
\hline
$\#$ Layers & $\#$Params of FCN & $\#$ Params of PIBiNN & Training NMSE (dB) & Test NMSE (dB)\\ [0.5 ex]
\hline \hline
$10$ & $1.5~\mathrm{B}$ & $50~\mathrm{K}$ & $-15.66$ & $-14.90$ \\
$20$ & $3~\mathrm{B}$ & $100~\mathrm{K}$ & $-19.66$ & $-18.22$ \\
\hline
\end{tabular}
}
\label{table_4}
\end{minipage}
\hfill
\begin{minipage}{0.4\textwidth}
\centering
\resizebox{\columnwidth}{!}{
\begin{tabular}{c | c| c | c }
\hline
Sparse Structure & Quantization & Training NMSE (dB) & Test NMSE (dB)\\ [0.5 ex]
\hline \hline
\textcolor{red}{\ding{55}}  & \textcolor{red}{\ding{55}} & $-22.16$ & $-17.06$ \\
\textcolor{green}{\ding{51}}& \textcolor{red}{\ding{55}}  & $-21.67$ & $-18.35$ \\
\textcolor{red}{\ding{55}} & \textcolor{green}{\ding{51}} & $-19.90$ & $-15.14$ \\
\textcolor{green}{\ding{51}} & \textcolor{green}{\ding{51}}& $-19.66$ & $-18.22$ \\
\hline
\end{tabular}
}
\label{table_16}
\end{minipage}
\vspace{-10pt}
\end{table*}
\begin{table}[h!] 
\centering 
\small
\caption{Test NMSE (dB) on the BSD500 and EEGdenoiseNet datasets across varying CS ratios using the one-bit DUN.} 

\resizebox{0.5\textwidth}{!}{ \begin{tabular}{c||c|c|c|c||c|c|c} \hline \multicolumn{1}{c||}{\textbf{Activation}} & \multicolumn{3}{c|}{\textbf{BSD500}} & \multicolumn{1}{c||}{} & \multicolumn{3}{c}{\textbf{EEGdenoiseNet}} \\ \cline{2-4} \cline{6-8} & $25\%$ & $50\%$ & $75\%$ & & $40\%$ & $50\%$ & $75\%$ \\ \hline\hline 
ST & $-8.76$ & $-15.31$ & $-17.47$ & & $-10.62$ & $-12.64$ & $-18.60$ \\ 
HT & $-11.81$ & $-13.59$ & $-14.79$ & & $-6.31$ & $-7.13$ & $-10.89$ \\ 
\hline 
\end{tabular} }
\label{table_14} 
\vspace{-10pt}
\end{table}
\begin{table}[h!]
\centering
\small
\caption{Training and Test NMSE comparison for large dataset across different schemes.}
\resizebox{0.5\textwidth}{!}{\begin{tabular}{lcc}
\toprule
Scheme & Training NMSE (dB) & Test NMSE (dB) \\
\midrule
Channel-wise & $-21.91$ & $-16.30$ \\
Ternary      & $-19.01$ & $-17.01$ \\
PIBiNN   & $-19.66$ & $-18.22$ \\
\bottomrule
\end{tabular} }
\label{tab:nmse_comparison}
\vspace{-10pt}
\end{table}
\begin{table}[h!]
\centering
\caption{Training and Test NMSE on BSD500 for different schemes at a $50\%$ CS rate using the ST activation.}
\resizebox{0.5\textwidth}{!}{\begin{tabular}{lcc}
\toprule
Scheme & Training NMSE (dB) & Test NMSE (dB) \\
\midrule
Channel-wise & $-17.21$ & $-15.46$ \\
Ternary      & $-16.44$ & $-15.96$ \\
PIBiNN   & $-16.64$ & $-16.29$ \\
\bottomrule
\end{tabular}}
\vspace{-10pt}
\label{tab:nmse_comparisonbsd}
\end{table}
$\bullet$ \textbf{BSD500 dataset:} We performed a CS experiment on natural image patches using the BSD500 dataset~\cite{martin2001database}, selecting $300$ images for training. From each image, we extracted $25$ random $8 \times 8$ patches, resulting in $6000$ training patches with zero-mean normalization and $1500$ test patches. Gaussian noise with standard deviation $0.05$ (assuming pixel values are normalized to the range $[0, 1]$) was added to the patches. The Discrete Cosine Transform (DCT) was then applied using a matrix \(\mbD \in \mathbb{R}^{64 \times 64}\), followed by applying a Gaussian sensing matrix \(\boldsymbol{\Phi} \in \mathbb{R}^{m \times 64}\) to the obtained DCT coefficients of the noisy patches. The resulting compressed measurements were passed to the one-bit DUN for image reconstruction.\\
$\bullet$ \textbf{EEGdenoiseNet dataset:} This dataset includes $4514$ clean EEG segments, $3400$ pure electrooculography segments, and $5598$ pure electromyography artifact segments, enabling the synthesis of contaminated EEG signals with known ground-truth clean components \cite{zhang2021eegdenoisenet}. We generated contaminated EEG signals using the model $\mbx_n=\mbx+\kappa \mbn$, where $\mbx$ is the clean EEG signal, $\mbn$ denotes ocular or myogenic artifacts, $\mbx_n$ is the resulting contaminated signal, and $\kappa$ is a hyperparameter to control the SNR. We applied the DCT with $\mbD \in \mathbb{R}^{512 \times 512}$, followed by Gaussian sensing, as in the BSD500 setup.
According to the results in Table~\ref{table_14}, the one-bit DUN successfully reconstructs images and EEG signals across various CS ratios ($m/64$ for images $m/512$), with the ST activation function consistently yielding better performance. Notably, the one-bit DUN achieves nearly the same reconstruction quality as the high-resolution DUN using the same number of layers, with only about a $1$ dB difference in NMSE.

$\bullet$ \textbf{Comparison results:} In Tables~\ref{tab:nmse_comparison} and \ref{tab:nmse_comparisonbsd}, we compare three compression methods: channel-wise binarization, ternary (1.58-bit) quantization, and the PIBiNN, reporting the final training and test results on both the large-model dataset and the BSD500 dataset. For the large-model case, we evaluate networks with $20$ layers and for Table~\ref{tab:nmse_comparisonbsd}, all settings follow those of Table~\ref{table_14}, except that the sensing matrix is chosen as a block-sparse matrix of the form $[\bPhi_1 \ \mathbf{0}; \ \mathbf{0} \ \bPhi_2] \in \mathbb{R}^{32 \times 64}$, where $\bPhi_1 $ and $\bPhi_2$ are Gaussian matrices of size $16 \times 32$. Setting the diagonal matrices equal would yield higher compression, but here we choose them differently to distinguish this experiment from the previous one.

In our scheme, sparsity arises directly from the problem physics, yielding sparse block-structured weights for the DUN, whereas ternary quantization imposes sparsification on the dense DUN via its own operator, and channel-wise binarization ignores sparsity entirely. From Tables~\ref{tab:nmse_comparison} and \ref{tab:nmse_comparisonbsd}, channel-wise quantization shows the best training accuracy, but our method generalizes better than both ternary and channel-wise binarization. The key reason is that our approach preserves all signs, maintaining the solver’s spectral geometry, while zeros emerge naturally from the architecture. In contrast, ternary sets weights in $(-0.5, 0.5)$ to zero at each epoch, potentially discarding \emph{small but essential} couplings that help LISTA-style layers coordinate corrections. By keeping those weak but informative links alive and only removing connections irrelevant by physics, our method achieves superior performance. Moreover, it does so with $~2\times$ lower memory than ternary (1 bit vs. ~2 bits/weight, with no mask or index overhead). The PIBiNN achieves a $99\%$ parameter reduction for the large model, compared to $78\%$ for ternary. On the BSD dataset (Table~\ref{tab:nmse_comparisonbsd}), the reductions are $50\%$ and $56\%$, respectively. The overlap between ternary-induced sparsity and the problem-driven sparse structure is reported in Appendix~\ref{sss}. Importantly, our approach uses a single global scale across all layers, whereas ternary requires a separate scale for each channel.

To evaluate the impact of $\delta$ on generalization and training/test performance, we conduct extensive experiments (Appendix~\ref{App_F4}). These results inform our discussion of practical considerations for selecting $\delta$. For completeness, we provide the limitations of our work in Appendix~\ref{I}.

\bibliographystyle{unsrt}
\bibliography{references}

\newpage
\appendix
\onecolumn

\section{Assumptions for Theorem~\ref{theorem_1}}
\label{assum}
The result presented in Theorem~\ref{theorem_1} is established under the following assumptions:
    
    \textbf{Assumption~I:} The input $(\bxi,\mbb)$ is a pair of random variables where $\bxi\in\mathcal{Y}\subseteq\mathbb{R}^d$ and $\mbb\in\mathcal{B}\subseteq\mathbb{R}^n$. Assume $\mbb$ satisfies $\mathbb{E}\mbb\mbb^{\top}=\sigma_b^2\mbI$. Assume $\bxi$ and $\mbb$ are independent, and their joint distribution follows a probability measure $\mathbb{P}$.
    
    \textbf{Assumption~II:} In our settings, we assume that both matrices $\mbQ_{\phi}$ and $\mbQ^{\mathrm{opt}}$ are symmetric and either positive semi-definite or positive definite, with a positive definite structure on the support $\mathcal{S}$. Specifically, $\Tilde\mbQ_{\phi},\Tilde\mbQ^{\mathrm{opt}}\in\mathcal{H}^{s\times s}_{r,L}$, the space of symmetric positive definite matrices whose smallest and largest singular values are bounded by $r$ and $L$, respectively, with $r,L>0$. Based on this assumption, the exact minimizer of \eqref{a1} for each sample follows a $s$-sparse structure, where the $s$ non-zero values are obtained as
    $\mbx^{\mathrm{opt}}_{\mathcal{S}}(\mbQ_i^{\mathrm{opt}},\mbb_i)=
    \left(\Tilde\mbQ_i^{\mathrm{opt}}\right)^{-1}\bar\mbb_{i},~i\in[N]$.
    A similar argument holds for 
    $\mbQ_{\phi}$, yielding: $\mbx^{\mathrm{opt}}_{\mathcal{S}}(\mbQ_{i,\phi},\mbb_i)=
    \left(\Tilde\mbQ_{i,\phi}\right)^{-1}\bar\mbb_{i},~i\in[N]$.

\section{Proof of Theorem~\ref{theorem_1}}
\label{app_rad}
Inspired by Lemma~4.2 of~\cite{chen2020understanding}, we extend this result to the sparse quadratic program. The key step of our argument is the following lemma, while the remainder of the proof follows the same reasoning as in Theorem~C.2. of~\cite{chen2020understanding}.
\begin{lemma}
\label{lem:4.2}
For every $\phi\in\Phi$ and $\mbQ^{\mathrm{opt}}$ it holds that
\[
\mathbb{E}\|\Tilde\mbQ_{\phi} - \Tilde\mbQ^{\mathrm{opt}}\|_{\mathrm{F}}^2 \;\le\; \sigma_b^{-2} L^4\Big(\sqrt{\mathbb{E}\ell_{\theta,\phi}^2} + \mathrm{Cvg}_{\theta}(k)\Big)^2,
\]
where $\ell_{\theta,\phi}(\bxi,\mbb)=\| \mathrm{Alg}_\theta^k(\mbQ_{\phi}(\bxi),\mbb) - \mbx^{\mathrm{opt}}(\mbQ^{\mathrm{opt}}(\bxi),\mbb)\|$, and $\mathrm{Cvg}_{\theta}(k)$ denotes the convergence factor of the algorithm layer.
\end{lemma}

\begin{proof}
Let $\varepsilon := \mathbb{E}\ell_{\theta,\phi}^2$. For any $(\bxi,\mbb)$ we have
\begin{align}
\ell_{\theta,\phi}(\bxi,\mbb)
&\ge \|\mbx^{\mathrm{opt}}(\mbQ_{\phi}(\bxi),\mbb)-\mbx^{\mathrm{opt}}(\mbQ^{\mathrm{opt}}(\bxi),\mbb)\|
      - \|\mathrm{Alg}_\theta^k(\mbQ_{\phi}(\bxi),\mbb)-\mbx^{\mathrm{opt}}(\mbQ_{\phi}(\bxi),\mbb)\|,\\
&\ge \|\Tilde\mbQ_{\phi}(\bxi)^{-1}\bar\mbb - \Tilde\mbQ^{\mathrm{opt}}(\bxi)^{-1}\bar\mbb\| - \mathrm{Cvg}_{\theta}(k).
\end{align}
Rearranging gives
\begin{equation}\label{eq:inv-diff-bound}
\|\Tilde\mbQ_{\phi}(\bxi)^{-1}\bar\mbb - \Tilde\mbQ^{\mathrm{opt}}(\bxi)^{-1}\bar\mbb\| \le \ell_{\theta,\phi}(\bxi,\mbb) + \mathrm{Cvg}_{\theta}(k).
\end{equation}

Since $\mathbb{E}_{\mbb}\mbb \mbb^\top=\sigma_b^2 \mbI$, we also have $\mathbb{E}_{\bar\mbb}\bar\mbb \bar\mbb^\top=\sigma_b^2 \mbI$ which together with
\[
\mathbb{E}_{\bar\mbb}\|\Tilde\mbQ_{\phi}(\bxi)^{-1}\bar\mbb - \Tilde\mbQ^{\mathrm{opt}}(\bxi)^{-1}\bar\mbb\|
= \sigma_b^2\|\Tilde\mbQ_{\phi}(\bxi)^{-1}(\Tilde\mbQ_{\phi}(\bxi)-\Tilde\mbQ^{\mathrm{opt}}(\bxi))\Tilde\mbQ^{\mathrm{opt}}(\bxi)^{-1}\|_\mathrm{F}^2,
\]
and the spectral bounds $r \mbI \preceq \Tilde\mbQ_{\phi}(\bxi),\Tilde\mbQ^\mathrm{opt}(\bxi)\preceq L \mbI$, we obtain
\begin{align}
\|\Tilde\mbQ_{\phi}(\bxi)-\Tilde\mbQ^{\mathrm{opt}}(\bxi)\|_\mathrm{F}^2
&\le \sigma_b^{-2} L^4 \; \mathbb{E}_{\bar\mbb}\|\Tilde\mbQ_{\phi}(\bxi)^{-1}\bar\mbb - \Tilde\mbQ^{\mathrm{opt}}(\bxi)^{-1}\bar\mbb\|^2.
\end{align}

Combine this inequality with \eqref{eq:inv-diff-bound}, integrate (expectation over $(\bxi,\mbb)$),
and use $(\mathbb{E}\ell_{\theta,\phi})^2 \le \mathbb{E}\ell_{\theta,\phi}^2$ to get
\[
\mathbb{E}\|\Tilde\mbQ_{\phi} - \Tilde\mbQ^{\mathrm{opt}}\|_\mathrm{F}^2
\le \sigma_b^{-2} L^4\Big(\sqrt{\mathbb{E}\ell_{\theta,\phi}^2} + \mathrm{Cvg}_{\theta}(k)\Big)^2,
\]
which is the claimed bound.
\end{proof}

\section{Convergence Analysis Based on the Set of Good Weights in Definition~\ref{def_1}}
\label{App_A}
We consider $\delta=1$ in the update process of \eqref{update_sqp_1}.
To present the convergence result, at first, we should find a proper value for $\theta_k$ to guarantee that $\operatorname{supp}(\mbx_{k})=\mathcal{S}$ when $\mbW_k$ is chosen from $\mathcal{X}_{\mbW_k}(\mbQ)$. Let $\mbx_0=0$ and by induction assume that $\operatorname{supp}(\mbx_{k-1})=\mathcal{S}$. Then $\forall i\notin\mathcal{S}$, we can write
\begin{equation}
\label{cvg2}
\begin{aligned}
x_{i,k}&=\operatorname{ST}_{\theta_k}\left(x_{i,k-1}-\mbW_{i,k}^{\top}(\bar\mbQ\mbx_{k-1}-\bar\mbb)\right),\\&=\operatorname{ST}_{\theta_k}\left(-\mbW_{i,k}^{\top}(\bar\mbQ\mbx_{k-1}-\bar\mbb)\right).
\end{aligned}
\end{equation}
Following Appendix~\ref{assum},
we can rewrite \eqref{cvg2} as
\begin{equation}
\label{cvg3}
\begin{aligned}
x_{i,k}&=\operatorname{ST}_{\theta_k}\left(-\mbW_{i,k}^{\top}\left(\Tilde\mbQ\mbx_{\mathcal{S},k-1}-\Tilde\mbQ\mbx^{\mathrm{opt}}_{\mathcal{S}}(\mbQ,\mbb)\right)\right),\\&=\operatorname{ST}_{\theta_k}\left(-\sum_{j\in\mathcal{S}}\mbW_{i,k}^{\top}\Tilde\mbQ_j\left(x_{j,k-1}-x_j^{\mathrm{opt}}(\mbQ,\mbb)\right)\right).
\end{aligned}
\end{equation}
For $\mbW_k\in\mathcal{X}_{\mbW_k}(\mbQ)$, we set $\theta_k=\mu\sup_{\mbx^{\mathrm{opt}}(\mbQ,\mbb)}\|\mbx_{k-1}-\mbx^{\mathrm{opt}}(\mbQ,\mbb)\|_1$, where $\mu=\sup_{\mbQ}\mu_Q$. Then, we can write
\begin{equation}
\label{cvg4}
\theta_k\geq\mu\left\|\mbx_{k-1}-\mbx^{\mathrm{opt}}(\mbQ,\mbb)\right\|_1\geq\left|-\sum_{j\in\mathcal{S}}\mbW_{i,k}^{\top}\Tilde\mbQ_j\left(x_{j,k-1}-x_j^{\mathrm{opt}}(\mbQ,\mbb)\right)\right|,
\end{equation}
which implies $x_{i,k}=0$, $\forall i\notin\mathcal{S}$ by the definition of $\operatorname{ST}_{\theta_k}$.

Next, for $i\in\mathcal{S}$ we can write
\begin{equation}
\label{cvg5}
\begin{aligned}
x_{i,k}&=\operatorname{ST}_{\theta_k}\left(x_{i,k-1}-\mbW_{i,k}^{\top}\Tilde\mbQ(\mbx_{\mathcal{S},k-1}-\mbx_{\mathcal{S}}^{\mathrm{opt}}(\mbQ,\mbb))\right),\\&\in x_{i,k-1}-\mbW_{i,k}^{\top}\Tilde\mbQ(\mbx_{\mathcal{S},k-1}-\mbx_{\mathcal{S}}^{\mathrm{opt}}(\mbQ,\mbb))-\theta_k\partial\ell_1(z_{i,k-1}),
\end{aligned}
\end{equation}
where $z_{i,k-1}=x_{i,k-1}-\mbW_{i,k}^{\top}\Tilde\mbQ(\mbx_{\mathcal{S},k-1}-\mbx_{\mathcal{S}}^{\mathrm{opt}}(\mbQ,\mbb))$ and $\partial\ell_1(x)$ is the sub-gradient of $\ell_1$ norm. Since $\mbW_{i,k}^{\top}\Tilde\mbQ_i=1$, we have
\begin{equation}
\label{cvg6}
\begin{aligned}
x_{i,k-1}-\mbW_{i,k}^{\top}\Tilde\mbQ(\mbx_{\mathcal{S},k-1}-\mbx_{\mathcal{S}}^{\mathrm{opt}}(\mbQ,\mbb))&=x_{i,k-1}-\sum_{j\in\mathcal{S},j\neq i}\mbW_{i,k}^{\top}\Tilde\mbQ_j(x_{j,k-1}-x_{j}^{\mathrm{opt}}(\mbQ,\mbb))\\&-(x_{i,k-1}-x_i^{\mathrm{opt}}(\mbQ,\mbb)),\\&=x_i^{\mathrm{opt}}(\mbQ,\mbb)-\sum_{j\in\mathcal{S},j\neq i}\mbW_{i,k}^{\top}\Tilde\mbQ_j(x_{j,k-1}-x_{j}^{\mathrm{opt}}(\mbQ,\mbb)).
\end{aligned}
\end{equation}
Then,
\begin{equation}
\label{cvg7}
x_{i,k}-x_i^{\mathrm{opt}}(\mbQ,\mbb)\in -\sum_{j\in\mathcal{S},j\neq i}\mbW_{i,k}^{\top}\Tilde\mbQ_j(x_{j,k-1}-x_{j}^{\mathrm{opt}}(\mbQ,\mbb))-\theta_k\partial\ell_1(z_{i,k-1}),~\forall i\in\mathcal{S}.
\end{equation}
Note that every element in $\partial\ell_1(x)$ has a magnitude less than or equal to $1$. Therefore, for all $i\in\mathcal{S}$,
\begin{equation}
\label{cvg8}
\begin{aligned}
\left|x_{i,k}-x^{\mathrm{opt}}_i\right|&\leq\sum_{j\in\mathcal{S},j\neq i}\left|\mbW_{i,k}^{\top}\Tilde\mbQ_j\right|\left|x_{j,k-1}-x_j^{\mathrm{opt}}(\mbQ,\mbb)\right|+\theta_k,\\&\leq\mu_Q\sum_{j\in\mathcal{S},j\neq i}\left|x_{j,k-1}-x_j^{\mathrm{opt}}(\mbQ,\mbb)\right|+\theta_k.
\end{aligned}
\end{equation}
Based on the chosen value of $\theta_k$, we can write $\|\mbx_k-\mbx^{\mathrm{opt}}(\mbQ,\mbb)\|_1=\|\mbx_{\mathcal{S},k}-\mbx_{\mathcal{S}}^{\mathrm{opt}}(\mbQ,\mbb)\|_1$ for all $k$. Then
\begin{equation}
\label{cvg9}
\begin{aligned}
\|\mbx_{k}-\mbx^{\mathrm{opt}}(\mbQ,\mbb)\|_1&=\sum_{i\in\mathcal{S}}\left|x_{i,k}-x_i^{\mathrm{opt}}(\mbQ,\mbb)\right|,\\&\leq\sum_{i\in\mathcal{S}}\left(\mu_Q\sum_{j\in\mathcal{S},j\neq i}\left|x_{j,k-1}-x_j^{\mathrm{opt}}(\mbQ,\mbb)\right|+\theta_k\right),\\&=\mu_Q(|\mathcal{S}|-1)\sum_{i\in\mathcal{S}}\left|x_{i,k-1}-x_i^{\mathrm{opt}}(\mbQ,\mbb)\right|+\theta_k|\mathcal{S}|,\\&\leq\mu_Q(|\mathcal{S}|-1)\|\mbx_{k-1}-\mbx^{\mathrm{opt}}(\mbQ,\mbb)\|_{1}+\theta_k|\mathcal{S}|.
\end{aligned}
\end{equation}
To generalize \eqref{cvg9} for the whole dataset, we take the supremum over $\mbx^{\mathrm{opt}}(\mbQ,\mbb)$, and by $|\mathcal{S}|=s$ we can write
\begin{equation}
\label{cvg10}
\sup_{\mbx^{\mathrm{opt}}(\mbQ,\mbb)}\|\mbx_{k}-\mbx^{\mathrm{opt}}(\mbQ,\mbb)\|_1\leq\mu(s-1)\sup_{\mbx^{\mathrm{opt}}(\mbQ,\mbb)}\|\mbx_{k-1}-\mbx^{\mathrm{opt}}(\mbQ,\mbb)\|_{1}+s\theta_k,
\end{equation}
where by our choice of $\theta_k$, we have
\begin{equation}
\label{cvg11}
\sup_{\mbx^{\mathrm{opt}}(\mbQ,\mbb)}\|\mbx_{k}-\mbx^{\mathrm{opt}}(\mbQ,\mbb)\|_1\leq(2\mu s-\mu)\sup_{\mbx^{\mathrm{opt}}(\mbQ,\mbb)}\|\mbx_{k-1}-\mbx^{\mathrm{opt}}(\mbQ,\mbb)\|_{1}.
\end{equation}
By recursively applying the bound \eqref{cvg11}, we can obtain
\begin{equation}
\label{cvg12}
\begin{aligned}
\sup_{\mbx^{\mathrm{opt}}(\mbQ,\mbb)}\|\mbx_{k}-\mbx^{\mathrm{opt}}(\mbQ,\mbb)\|_1&\leq(2\mu s-\mu)^{k}\sup_{\mbx^{\mathrm{opt}}(\mbQ,\mbb)}\|\mbx_0-\mbx^{\mathrm{opt}}(\mbQ,\mbb)\|_{1}\\&=(2\mu s-\mu)^{k}\sup_{\mbx^{\mathrm{opt}}(\mbQ,\mbb)}\|\mbx^{\mathrm{opt}}(\mbQ,\mbb)\|_{1},\\&\leq(2\mu s-\mu)^{k}B.
\end{aligned}
\end{equation}
Since $\|\mbx\|_2\leq\|\mbx\|_1$ for any $\mbx\in\mathbb{R}^n$, we can rewrite \eqref{cvg12} as
\begin{equation}
\label{cvg13}
\sup_{\mbx^{\mathrm{opt}}(\mbQ,\mbb)}\|\mbx_{k}-\mbx^{\mathrm{opt}}(\mbQ,\mbb)\|\leq(2\mu s-\mu)^{k}B.
\end{equation}
As long as $s<\frac{1}{2}+\frac{1}{2\mu}$, the convergence bound \eqref{cvg13} holds uniformly for all $\mbx^{\mathrm{opt}}(\mbQ,\mbb)$. Therefore, the convergence parameter is
\begin{equation}
\label{cvg14}
\mathrm{Cvg}(k)=(2\mu s-\mu)^{k}B.
\end{equation}
$\bullet$ \textbf{Limitation:} As discussed in Section~\ref{sec5}, the limitation of this argument stems from the constraint $\mbW_i^{\top}\bar\mbQ_i = 1$ for all $i \in [n]$. In particular, it is not clear that binary weights can satisfy this condition, and no straightforward proof appears to exist.
\section{Convergence Analysis For Soft-Thresholding Operator}
\label{App_B}
In this section, we provide the convergence analysis of the ST operator for any $\delta\in(0,1]$. For each layer $k\in[K]$, we select the weight matrix $\mbW_k$ according to the construction detailed in Section~\ref{sec5}. Specifically, the weight matrices $\{\mbW_k\}_{k=1}^K$ are chosen to satisfy:
\begin{equation}
\label{to10}
\mbW_{k}\in\mathcal{ST}_{\mbW}(\mbQ;\delta)\triangleq\left\{\mbW\in\{-\lambda,\lambda\}^{s\times n}:\left\|\delta\mbI-\mbW_{\mathcal{S}}^{\top}\Tilde\mbQ\right\|+\mu_Q^{W} s\in(0,1)\right\},~k\in[K],~\delta\in(0,1],
\end{equation}
where similar to Definition~\ref{def_1}, $\mu_Q^W$ is defined as
\begin{equation}
\label{def_mu}
\mu_Q^W=\sup_{i,j\in[n],i\neq j}\left|\mbW_i^{\top}\bar\mbQ_j\right|.
\end{equation}
Similar to Appendix~\ref{App_A}, in order to guarantee that $\operatorname{supp}(\mbx_{k}) = \mathcal{S}$, we choose $\theta_k$ as
\begin{equation}
\label{to7}
\theta_k=\mu^{W_k}\sup_{\mbx^{\mathrm{opt}}(\mbQ,\mbb)}\|\mbx_{k-1}-\mbx^{\mathrm{opt}}(\mbQ,\mbb)\|_1,
\end{equation}
where $\mu^{W_k}=\sup_{\mbQ}\mu_Q^{W_k}$. We can then write the convergence result as follows:
\begin{equation}
\label{to8}
\begin{aligned}
\left\|\mbx_k-\mbx^{\mathrm{opt}}(\mbQ,\mbb)\right\|&=\left\|\operatorname{ST}_{\theta_k}\left(\delta\mbx_{k-1}-\mbW^{\top}_k\left(\bar\mbQ\mbx_{k-1}-\bar\mbb\right)\right)-\mbx^{\mathrm{opt}}(\mbQ,\mbb)\right\|,\\&=\left\|\delta\mbx_{\mathcal{S},k-1}-\mbW^{\top}_{\mathcal{S},k}\left(\Tilde\mbQ\mbx_{\mathcal{S},k-1}-\bar\mbb\right)-\mbx^{\mathrm{opt}}_{\mathcal{S}}(\mbQ,\mbb)-\theta_k\operatorname{sign}(\mathcal{F}_{k-1})\right\|,~k\in[K],
\end{aligned}
\end{equation}
where the second step follows from the selection of $\theta_k$ in \eqref{to7} and
\begin{equation}
\label{boz}
\mathcal{F}_{k-1}=\delta\mbx_{\mathcal{S},k-1}-\mbW^{\top}_{\mathcal{S},k}\left(\Tilde\mbQ\mbx_{\mathcal{S},k-1}-\bar\mbb\right),~k\in[K].
\end{equation}
We can expand \eqref{to8} as \par\noindent\small
\begin{equation}
\label{to1}
\begin{aligned}
\left\|\mbx_k-\mbx^{\mathrm{opt}}(\mbQ,\mbb)\right\|&=\left\|\delta\mbx_{\mathcal{S},k-1}-\mbW^{\top}_{\mathcal{S},k}\left(\Tilde\mbQ\mbx_{\mathcal{S},k-1}-\bar\mbb\right)-\mbx^{\mathrm{opt}}_{\mathcal{S}}(\mbQ,\mbb)-\theta_k\operatorname{sign}(\mathcal{F}_{k-1})\right\|,\\&=\left\|\delta\mbx_{\mathcal{S},k-1}-\mbW^{\top}_{\mathcal{S},k}\Tilde\mbQ\left(\mbx_{\mathcal{S},k-1}-\mbx_{\mathcal{S}}^{\mathrm{opt}}(\mbQ,\mbb)\right)-\delta\mbx_{\mathcal{S}}^{\mathrm{opt}}(\mbQ,\mbb)-(1-\delta)\mbx_{\mathcal{S}}^{\mathrm{opt}}(\mbQ,\mbb)-\theta_k\operatorname{sign}(\mathcal{F}_{k-1})\right\|,\\&\leq\left\|\left(\delta\mbI-\mbW^{\top}_{\mathcal{S},k}\Tilde\mbQ\right)
\left(\mbx_{\mathcal{S},k-1}-\mbx_{\mathcal{S}}^{\mathrm{opt}}(\mbQ,\mbb)\right)\right\|+(1-\delta)\left\|\mbx^{\mathrm{opt}}(\mbQ,\mbb)\right\|+\theta_k\sqrt{s},\\&\leq\left\|\left(\delta\mbI-\mbW^{\top}_{\mathcal{S},k}\Tilde\mbQ\right)
\left(\mbx_{\mathcal{S},k-1}-\mbx_{\mathcal{S}}^{\mathrm{opt}}(\mbQ,\mbb)\right)\right\|+(1-\delta)B+\theta_k\sqrt{s},~k\in[K].
\end{aligned}
\end{equation}\normalsize
Based on the selection of $\theta_k$ in \eqref{to7}, we have
\begin{equation}
\begin{aligned}
\label{to3}
\left\|\mbx_k-\mbx^{\mathrm{opt}}(\mbQ,\mbb)\right\|&\leq\left\|\delta\mbI-\mbW^{\top}_{\mathcal{S},k}\Tilde\mbQ\right\|\left\|\mbx_{k-1}-\mbx^{\mathrm{opt}}(\mbQ,\mbb)\right\|+(1-\delta)B\\&+\mu^{W_k}\sqrt{s}\sup_{\mbx^{\mathrm{opt}}(\mbQ,\mbb)}\|\mbx_{k-1}-\mbx^{\mathrm{opt}}(\mbQ,\mbb)\|_1,\\&\leq\left\|\delta\mbI-\mbW^{\top}_{\mathcal{S},k}\Tilde\mbQ\right\|\left\|\mbx_{k-1}-\mbx^{\mathrm{opt}}(\mbQ,\mbb)\right\|+(1-\delta)B\\&+\mu^{W_k} s\sup_{\mbx^{\mathrm{opt}}(\mbQ,\mbb)}\|\mbx_{k-1}-\mbx^{\mathrm{opt}}(\mbQ,\mbb)\|,~k\in[K].
\end{aligned}
\end{equation}
To generalize \eqref{to3} for the whole dataset, we take the supremum over $\mbx^{\mathrm{opt}}(\mbQ,\mbb)$:
\noindent\small
\begin{equation}
\begin{aligned}
\label{to3_new}
\sup_{\mbx^{\mathrm{opt}}(\mbQ,\mbb)}\left\|\mbx_k-\mbx^{\mathrm{opt}}(\mbQ,\mbb)\right\|&\leq\sup_{\Tilde\mbQ\in\mathcal{H}^{s\times s}_{r,L}}\left\|\delta\mbI-\mbW^{\top}_{\mathcal{S},k}\Tilde\mbQ\right\|\sup_{\mbx^{\mathrm{opt}}(\mbQ,\mbb)}\left\|\mbx_{k-1}-\mbx^{\mathrm{opt}}(\mbQ,\mbb)\right\|+(1-\delta)B\\&+\mu^{W_k} s\sup_{\mbx^{\mathrm{opt}}(\mbQ,\mbb)}\|\mbx_{k-1}-\mbx^{\mathrm{opt}}(\mbQ,\mbb)\|,\\&=\left[\sup_{\Tilde\mbQ\in\mathcal{H}^{s\times s}_{r,L}}\left\|\delta\mbI-\mbW^{\top}_{\mathcal{S},k}\Tilde\mbQ\right\|+\mu^{W_k} s\right]\sup_{\mbx^{\mathrm{opt}}(\mbQ,\mbb)}\left\|\mbx_{k-1}-\mbx^{\mathrm{opt}}(\mbQ,\mbb)\right\|+(1-\delta)B,~k\in[K].
\end{aligned}
\end{equation} \normalsize
By recursively applying the bound \eqref{to3_new}, we can obtain
\begin{equation}
\label{to4}
\begin{aligned}
\sup_{\mbx^{\mathrm{opt}}(\mbQ,\mbb)}\left\|\mbx_k-\mbx^{\mathrm{opt}}(\mbQ,\mbb)\right\|&\leq\Pi_{i=1}^{k}\left[ \sup_{\Tilde\mbQ\in\mathcal{H}^{s\times s}_{r,L}}\left\|\delta\mbI-\mbW^{\top}_{\mathcal{S},i}\Tilde\mbQ\right\|+\mu^{W_i} s\right]\sup_{\mbx^{\mathrm{opt}}(\mbQ,\mbb)}\left\|\mbx_{0}-\mbx^{\mathrm{opt}}(\mbQ,\mbb)\right\|\\&+(1-\delta)B+\sum_{i=1}^{k-1}(1-\delta)B\Pi_{j=1}^{i}\left[\sup_{\Tilde\mbQ\in\mathcal{H}^{s\times s}_{r,L}}\left\|\delta\mbI-\mbW^{\top}_{\mathcal{S},k-i}\Tilde\mbQ\right\|+\mu^{W_i} s\right],~k\in[K].
\end{aligned}
\end{equation}
Define $\alpha_{\theta}\triangleq\sup_{i\in[k]}\left[ \sup_{\Tilde\mbQ\in\mathcal{H}^{s\times s}_{r,L}}\left\|\delta\mbI-\mbW^{\top}_{\mathcal{S},i}\Tilde\mbQ\right\|+\mu^{W_i} s\right]$. With this definition, we can rewrite \eqref{to4} as
\begin{equation}
\label{to5}
\begin{aligned}
\sup_{\mbx^{\mathrm{opt}}(\mbQ,\mbb)}\left\|\mbx_k-\mbx^{\mathrm{opt}}(\mbQ,\mbb)\right\|&\leq\alpha_{\theta}^kB+\sum_{i=0}^{k-1}(1-\delta)B\alpha_{\theta}^i\\&\leq\mathrm{Cvg}_{\theta}(k),~k\in[K],
\end{aligned}
\end{equation}
where
\begin{equation}
\label{to6}
\mathrm{Cvg}_{\theta}(k)=\alpha_{\theta}^kB+\underbrace{\frac{1-\delta}{1-\alpha_{\theta}}B}_{u(\delta)},~k\in[K],
\end{equation}
where the worst-case convergence is $\mathrm{Cvg}(k)=\sup_{\theta}\mathrm{Cvg}_{\theta}(k)$.\\
$\bullet$ \textbf{Discussion on the effect of $\delta$:}
In \eqref{to6}, $u(\delta)$ attains its minimum at $\delta=1$, where $u(\delta=1)=0$. Thus, we expect binary network convergence to degrade as $\delta$ decreases, with the best convergence at $\delta=1$. In Section~\ref{sec5}, we established the non-emptiness of $\mathcal{ST}_{\mbW}(\mbQ;\delta)$ in \eqref{to10} by selecting $(\delta,\lambda)$ such that the upper bound of $\left[\left\|\delta\mbI-\mbW_{\mathcal{S}}^{\top}\Tilde\mbQ\right\|+\mu_Q^W s\right]$ is strictly less than $1$ for sufficiently small sparsity level $s$. This proof, however, does not cover the case $\delta=1$. In Section~\ref{App_F4}, we provide numerical evidence that optimization methods such as Adam generally do not produce weights lying in the set \eqref{to10} when $\delta=1$. In contrast, when $\delta\in(0,1)$ is chosen appropriately, the condition in \eqref{to10} is satisfied, ensuring the generalization guarantees discussed in Section~\ref{sec5}, albeit with degraded convergence. Nonetheless, for sufficiently small sparsity $s$, we also observe that Adam can still generate weights that belong to \eqref{to10} even in the case $\delta=1$. As part of future work, we aim to tighten this bound so that it remains valid for larger sparsity levels $s$ in the case $\delta=1$.
\section{Convergence Analysis For Hard-Thresholding Operator}
\label{App_C}
In this section, we present the convergence analysis of the HT operator for any $\delta\in(0,1]$. For each layer $k\in[K]$, the weight matrices $\{\mbW_k\}_{k=1}^K$ are chosen to satisfy:
\begin{equation}
\label{ht1}
\mbW_{k}\in\mathcal{HT}_{\mbW}(\mbQ;\delta)\triangleq\left\{\mbW\in\{-\lambda,\lambda\}^{s\times n}:\left\|\delta\mbI-\mbW_{\mathcal{S}}^{\top}\Tilde\mbQ\right\|\in(0,1)\right\},~k\in[K],~\delta\in(0,1].
\end{equation} 
Similar to Appendix~\ref{App_A}, in order to guarantee that $\operatorname{supp}(\mbx_{k}) = \mathcal{S}$, we choose $\theta_k$ as
\begin{equation}
\label{s_teta}
\theta_k=\mu^{W_k}\sup_{\mbx^{\mathrm{opt}}(\mbQ,\mbb)}\|\mbx_{k-1}-\mbx^{\mathrm{opt}}(\mbQ,\mbb)\|_1,
\end{equation}
where $\mu^{W_k}$ is defined as in Appendix~\ref{App_B}.
We can then write the convergence result as follows:
\begin{equation}
\label{ht2}
\begin{aligned}
\left\|\mbx_k-\mbx^{\mathrm{opt}}(\mbQ,\mbb)\right\|&=\left\|\operatorname{HT}_{\theta_k}\left(\delta\mbx_{k-1}-\mbW^{\top}_k\left(\bar\mbQ\mbx_{k-1}-\bar\mbb\right)\right)-\mbx^{\mathrm{opt}}(\mbQ,\mbb)\right\|,\\&=\left\|\delta\mbx_{\mathcal{S},k-1}-\mbW_{\mathcal{S},k}^{\top}\left(\Tilde\mbQ\mbx_{\mathcal{S},k-1}-\bar\mbb\right)-\mbx^{\mathrm{opt}}_{\mathcal{S}}(\mbQ,\mbb)\right\|,
\end{aligned}
\end{equation}
where the second step follows from the selection of $\theta_k$ in \eqref{s_teta}. We can expand \eqref{ht2} as
\begin{equation}
\label{ht2_new}
\begin{aligned}
\left\|\mbx_k-\mbx^{\mathrm{opt}}(\mbQ,\mbb)\right\|&=\left\|\delta\mbx_{\mathcal{S},k-1}-\mbW_{\mathcal{S},k}^{\top}\left(\Tilde\mbQ\mbx_{\mathcal{S},k-1}-\bar\mbb\right)-\mbx^{\mathrm{opt}}_{\mathcal{S}}(\mbQ,\mbb)\right\|,\\&=\left\|\delta\mbx_{\mathcal{S},k-1}-\mbW^{\top}_{\mathcal{S},k}\Tilde\mbQ\left(\mbx_{\mathcal{S},k-1}-\mbx_{\mathcal{S}}^{\mathrm{opt}}(\mbQ,\mbb)\right)-\delta\mbx_{\mathcal{S}}^{\mathrm{opt}}(\mbQ,\mbb)-(1-\delta)\mbx_{\mathcal{S}}^{\mathrm{opt}}(\mbQ,\mbb)\right\|,\\&\leq\left\|\left(\delta\mbI-\mbW^{\top}_{\mathcal{S},k}\Tilde\mbQ\right)
\left(\mbx_{\mathcal{S},k-1}-\mbx_{\mathcal{S}}^{\mathrm{opt}}(\mbQ,\mbb)\right)\right\|+(1-\delta)\left\|\mbx^{\mathrm{opt}}(\mbQ,\mbb)\right\|,\\&\leq\left\|\left(\delta\mbI-\mbW^{\top}_{\mathcal{S},k}\Tilde\mbQ\right)
\left(\mbx_{\mathcal{S},k-1}-\mbx_{\mathcal{S}}^{\mathrm{opt}}(\mbQ,\mbb)\right)\right\|+(1-\delta)B,~k\in[K].
\end{aligned}
\end{equation}
To generalize \eqref{ht2_new} for the whole dataset, we take the supremum over $\mbx^{\mathrm{opt}}(\mbQ,\mbb)$:
\begin{equation}
\label{gen}
\sup_{\mbx^{\mathrm{opt}}(\mbQ,\mbb)}\left\|\mbx_k-\mbx^{\mathrm{opt}}(\mbQ,\mbb)\right\|\leq\sup_{\Tilde\mbQ\in\mathcal{H}^{s\times s}_{r,L}}\left\|\delta\mbI-\mbW^{\top}_{\mathcal{S},k}\Tilde\mbQ\right\|\sup_{\mbx^{\mathrm{opt}}(\mbQ,\mbb)}\left\|\mbx_{k-1}-\mbx^{\mathrm{opt}}(\mbQ,\mbb)\right\|+(1-\delta)B,~k\in[K].
\end{equation}
By recursively applying the bound \eqref{gen}, we can obtain
\begin{equation}
\label{ht3}
\begin{aligned}
\sup_{\mbx^{\mathrm{opt}}(\mbQ,\mbb)}\left\|\mbx_k-\mbx^{\mathrm{opt}}(\mbQ,\mbb)\right\|&\leq\Pi_{i=1}^{k}\left[\sup_{\Tilde\mbQ\in\mathcal{H}^{s\times s}_{r,L}}\left\|\delta\mbI-\mbW^{\top}_{\mathcal{S},i}\Tilde\mbQ\right\|\right]\sup_{\mbx^{\mathrm{opt}}(\mbQ,\mbb)}\left\|\mbx_{0}-\mbx^{\mathrm{opt}}(\mbQ,\mbb)\right\|\\&+(1-\delta)B+\sum_{i=1}^{k-1}(1-\delta)B\Pi_{j=1}^{i}\left[\sup_{\Tilde\mbQ\in\mathcal{H}^{s\times s}_{r,L}}\left\|\delta\mbI-\mbW^{\top}_{\mathcal{S},k-i}\Tilde\mbQ\right\|\right],~k\in[K].
\end{aligned}
\end{equation}
Define $\alpha_{\theta}\triangleq\sup_{i\in[k]}\left[\sup_{\Tilde\mbQ\in\mathcal{H}^{s\times s}_{r,L}}\left\|\delta\mbI-\mbW^{\top}_{\mathcal{S},i}\Tilde\mbQ\right\|\right]$. With this definition, we can rewrite \eqref{ht3} as
\begin{equation}
\label{ht4}
\begin{aligned}
\sup_{\mbx^{\mathrm{opt}}(\mbQ,\mbb)}\left\|\mbx_k-\mbx^{\mathrm{opt}}(\mbQ,\mbb)\right\|&\leq\alpha_{\theta}^kB+\sum_{i=0}^{k-1}(1-\delta)B\alpha_{\theta}^i\\&\leq\mathrm{Cvg}_{\theta}(k),~k\in[K],
\end{aligned}
\end{equation}
where
\begin{equation}
\label{ht5}
\mathrm{Cvg}_{\theta}(k)=\alpha_{\theta}^kB+\underbrace{\frac{1-\delta}{1-\alpha_{\theta}}B}_{v(\delta)},~k\in[K],
\end{equation}
where the worst-case convergence is $\mathrm{Cvg}(k)=\sup_{\theta}\mathrm{Cvg}_{\theta}(k)$.\\
$\bullet$ \textbf{Discussion on the effect of $\delta$:} 
Similar to the ST update, for the HT update the function $v(\delta)$ attains its minimum at $\delta=1$, where $v(\delta=1)=0$. Thus, we expect the convergence of binary networks under the HT update to degrade as $\delta$ decreases. For the non-emptiness of $\mathcal{HT}_{\mbW}(\mbQ;\delta)$ in \eqref{ht1}, see the argument in Section~\ref{sec5}, where suitable choices of $\delta$ and $\lambda$ were made; however, that argument does not cover the case $\delta=1$. To establish non-emptiness in the specific case $\delta=1$, we can impose a sparse structure on the weight matrices $\mbW_{\mathcal{S},k}$. Specifically, let $\mbW_{\mathcal{S},k}$ be diagonal binary matrices with diagonal entries $\lambda$ or $-\lambda$. If $\mbW_k=\operatorname{diag}([\lambda\cdots\lambda])$ and $\lambda<\frac{1}{L}$, then $\mathcal{HT}_{\mbW}(\mbQ;1)$ is non-empty since $\Tilde\mbQ\in\mathcal{H}^{s\times s}_{r,L}$. Under this construction, the learning process can provably converge with the rate $\mathrm{Cvg}_{\theta}(k)=\alpha_{\theta}^kB$.
\section{Stability Analysis For (Soft/Hard)-Thresholding Operators}
\label{App_D}
In this section, we analyze and derive the stability parameters for both ST and HT update processes.
\subsection{Soft-Thresholding}
\label{App_D1}
By considering the parameters $\left(\mbQ^{\prime},\mbb^{\prime}\right)$, we define $\mbx^{\prime}_k$ as
\begin{equation}
\label{stab_n}
\mbx^{\prime}_k = \operatorname{ST}_{\theta_k}\left(\delta\mbx^{\prime}_{k-1}-\mbW_k^{\top}\left(\bar\mbQ^{\prime}\mbx^{\prime}_{k-1}-\bar\mbb^{\prime}\right)\right),~k\in[K],
\end{equation}
where $\left(\mbQ^{\prime},\mbb^{\prime}\right)$ represents a perturbation of $(\mbQ, \mbb)$ in the dataset.
We assume that $\mbW_k \in \mathcal{ST}_{\mbW}(\mbQ;\delta)$ for all $k \in [K]$ and for all $\mbQ$ in the dataset. 
Furthermore, we consider a perturbation pair $\left(\mbQ^{\prime}, \mbb^{\prime}\right)$ that satisfies the following two conditions: (i) for all $k \in [K]$, the matrices $\mbW_k$ also belong to $\mathcal{ST}_{\mbW}(\mbQ^{\prime};\delta)$ such that
\begin{equation}
\begin{aligned}
\left\|\delta\mbI-\mbW_{\mathcal{S},k}^{\top}\Tilde\mbQ^{\prime}\right\|+\mu_{Q^{\prime}}^{W_k}s\leq\sup_{\Tilde\mbQ\in\mathcal{H}^{s\times s}_{r,L}}\left\|\delta\mbI-\mbW_{\mathcal{S},k}^{\top}\Tilde\mbQ\right\|+\mu^{W_k} s\in(0,1),
\end{aligned}
\end{equation}
and (ii) the threshold parameter $\theta_k$ in \eqref{to7} determines the support $\mathcal{S}$ of $\mbx^{\prime}_{k}$ at each iteration $k \in [K]$.
By the definition of $\mbx_k$ in \eqref{update_sqp_1}, 
we can write the second norm of $\mbx_k$ for $\delta\in(0,1]$ as
\begin{equation}
\label{stab0}
\begin{aligned}
\|\mbx_k\|&=\left\|\operatorname{ST}_{\theta_k}\left(\delta\mbx_{k-1}-\mbW_k^{\top}\left(\bar\mbQ\mbx_{k-1}-\bar\mbb\right)\right)\right\|,~k\in[K],
\end{aligned}
\end{equation}
where by choosing $\theta_k$ as \eqref{to7}, we can write
\begin{equation}
\label{goraz_1}
\left\|\mbx_k\right\|=\left\|\delta\mbx_{\mathcal{S},k-1}-\mbW^{\top}_{\mathcal{S},k}\left(\Tilde\mbQ\mbx_{\mathcal{S},k-1}-\bar\mbb\right)-\theta_k\operatorname{sign}(\mathcal{F}_{k-1})\right\|,~k\in[K],
\end{equation}
where $\mathcal{F}_{k-1}$ is defined as \eqref{boz}. We can further simplify \eqref{goraz_1} as
\begin{equation}
\label{goraz_2}
\begin{aligned}
\left\|\mbx_k\right\|&=\left\|\delta\mbx_{\mathcal{S},k-1}-\mbW^{\top}_{\mathcal{S},k}\left(\Tilde\mbQ\mbx_{\mathcal{S},k-1}-\bar\mbb\right)-\theta_k\operatorname{sign}(\mathcal{F}_{k-1})\right\|,\\&=\left\|\delta\mbx_{\mathcal{S},k-1}-\mbW^{\top}_{\mathcal{S},k}\Tilde\mbQ\left(\mbx_{\mathcal{S},k-1}-\mbx_{\mathcal{S}}^{\mathrm{opt}}\left(\mbQ,\mbb\right)\right)-\delta\mbx_{\mathcal{S}}^{\mathrm{opt}}\left(\mbQ,\mbb\right)+\delta\mbx_{\mathcal{S}}^{\mathrm{opt}}\left(\mbQ,\mbb\right)-\theta_k\operatorname{sign}(\mathcal{F}_{k-1})\right\|,\\&\leq\left\|\delta\mbI-\mbW^{\top}_{\mathcal{S},k}\Tilde\mbQ\right\|\left\|\mbx_{k-1}-\mbx^{\mathrm{opt}}\left(\mbQ,\mbb\right)\right\|+\delta B+\theta_k\sqrt{s},~k\in[K].
\end{aligned}
\end{equation}
To generalize \eqref{goraz_2} for the whole dataset, we take the supremum over $\mbx^{\mathrm{opt}}(\mbQ,\mbb)$:
\begin{equation}
\label{goraz_3}
\begin{aligned}
\sup_{\mbx^{\mathrm{opt}}(\mbQ,\mbb)}\left\|\mbx_k\right\|&\leq\left[\sup_{\Tilde\mbQ\in\mathcal{H}^{s\times s}_{r,L}}\left\|\delta\mbI-\mbW^{\top}_{\mathcal{S},k}\Tilde\mbQ\right\|+\mu^{W_k} s\right]\sup_{\mbx^{\mathrm{opt}}(\mbQ,\mbb)}\left\|\mbx_{k-1}-\mbx^{\mathrm{opt}}(\mbQ,\mbb)\right\|+\delta B,\\&\leq\underbrace{\left[\sup_{\Tilde\mbQ\in\mathcal{H}^{s\times s}_{r,L}}\left\|\delta\mbI-\mbW^{\top}_{\mathcal{S},k}\Tilde\mbQ\right\|+\mu^{W_k} s\right]\mathrm{Cvg}_{\theta}(k-1)+\delta B}_{\zeta_{\theta}(k)},~k\in[K].
\end{aligned}
\end{equation}
In Appendix~\ref{App_B}, we have shown that $\operatorname{Cvg}_{\theta}(k) = \alpha^{k}_{\theta}+\mathrm{const}_{(1)}$. Thus,
\begin{equation}
\zeta_{\theta}(k)=\left(\alpha^{k-1}_{\theta}+\mathrm{const}_{(1)}\right)\alpha_{\theta}+\mathrm{const}_{(2)}=\alpha^{k}_{\theta}+\mathrm{const}_{(1)}\alpha_{\theta}+\mathrm{const}_{(2)}.
\end{equation}

Define $\mathcal{F}^{\prime}_{k-1}$ analogously to $\mathcal{F}_{k-1}$ in \eqref{boz}. Similarly, let $\mbz_{k-1}$ and $\mbz^{\prime}_{k-1}$ be defined in the same manner as $\mathcal{F}_{k-1}$ and $\mathcal{F}^{\prime}_{k-1}$, but without restricting to the support $\mathcal{S}$. Under the stated assumptions, we have
\begin{equation}
\label{stab1}
\begin{aligned}
\left\|\mbx_k-\mbx^{\prime}_k\right\|&=\left\|\operatorname{ST}_{\theta_k}\left(\mbz_{k-1}\right)-\operatorname{ST}_{\theta_k}\left(\mbz^{\prime}_{k-1}\right)\right\|,\\&=\left\|\mathcal{F}_{k-1}-\theta_k\operatorname{sign}\left(\mathcal{F}_{k-1}\right)-\mathcal{F}^{\prime}_{k-1}+\theta_k\operatorname{sign}\left(\mathcal{F}^{\prime}_{k-1}\right)\right\|,\\&\leq\underbrace{\left\|\mathcal{F}_{k-1}-\mathcal{F}^{\prime}_{k-1}\right\|}_{\text{Term}\star}+\underbrace{\theta_k\left\|\operatorname{sign}\left(\mathcal{F}_{k-1}\right)-\operatorname{sign}\left(\mathcal{F}^{\prime}_{k-1}\right)\right\|}_{\text{Term}\star\star}.
\end{aligned}
\end{equation}
$\text{Term}\star\star$ can be bounded as
\begin{equation}
\label{goraz_4}
\begin{aligned}
\text{Term}\star\star&\leq2\sqrt{s}\theta_k,\\&\leq2\mu^{W_k} s\sup_{\mbx^{\mathrm{opt}}(\mbQ,\mbb)}\left\|\mbx_{k-1}-\mbx^{\mathrm{opt}}(\mbQ,\mbb)\right\|,\\&\leq2\mu^{W_k} s\mathrm{Cvg}_{\theta}(k-1).
\end{aligned}
\end{equation}
For $\text{Term}\star$, we obtain
\par\noindent\small
\begin{equation}
\label{goraz_5}
\begin{aligned}
\text{Term}\star&=\left\|\delta\left(\mbx_{\mathcal{S},k-1}-\mbx^{\prime}_{\mathcal{S},k-1}\right)-\mbW^{\top}_{\mathcal{S},k}\left(\Tilde\mbQ\mbx_{\mathcal{S},k-1}-\bar\mbb\right)+\mbW^{\top}_{\mathcal{S},k}\left(\Tilde\mbQ^{\prime}\mbx_{\mathcal{S},k-1}^{\prime}-\bar\mbb^{\prime}\right)\right\|,\\&=\left\|\delta\left(\mbx_{\mathcal{S},k-1}-\mbx^{\prime}_{\mathcal{S},k-1}\right)-\mbW^{\top}_{\mathcal{S},k}\left(\Tilde\mbQ\mbx_{\mathcal{S},k-1}-\bar\mbb\right)+\mbW^{\top}_{\mathcal{S},k}\left(\Tilde\mbQ^{\prime}\mbx_{\mathcal{S},k-1}^{\prime}-\bar\mbb^{\prime}\right)+\mbW^{\top}_{\mathcal{S},k}\Tilde\mbQ^{\prime}\mbx_{\mathcal{S},k-1}-\mbW^{\top}_{\mathcal{S},k}\Tilde\mbQ^{\prime}\mbx_{\mathcal{S},k-1}\right\|,\\&\leq\left\|\delta\mbI-\mbW^{\top}_{\mathcal{S},k}\Tilde\mbQ^{\prime}\right\|\left\|\mbx_{\mathcal{S},k-1}-\mbx^{\prime}_{\mathcal{S},k-1}\right\|+\left\|\mbW_{\mathcal{S},k}\right\|\left\|\bar\mbb-\bar\mbb^{\prime}\right\|+\left\|\mbW_{\mathcal{S},k}\right\|\left\|\Tilde\mbQ-\Tilde\mbQ^{\prime}\right\|\zeta_{\theta}(k-1).
\end{aligned}
\end{equation} \normalsize
By combining \eqref{goraz_4} and \eqref{goraz_5}, and under the assumption that $\mbx_0=\mbx^{\prime}_0$, we obtain
\begin{equation}
\label{stab2}
\begin{aligned}
&\|\mbx_{k}-\mbx^{\prime}_k\|\leq \left(\|\mbW_{\mathcal{S},k}\|+\sum^{k-1}_{i=1}\|\mbW_{\mathcal{S},k-i}\|\Pi^{i}_{j=1}\left\|\delta\mbI-\mbW^{\top}_{\mathcal{S},k-j+1}\Tilde\mbQ^{\prime}\right\|\right)\|\bar\mbb-\bar\mbb^{\prime}\|\\&+\left(\|\mbW_{\mathcal{S},k}\|\zeta_{\theta}(k-1)+\sum^{k-1}_{i=1}\|\mbW_{\mathcal{S},k-i}\|\zeta_{\theta}(k-i-1)\Pi^{i}_{j=1}\left\|\delta\mbI-\mbW^{\top}_{\mathcal{S},k-j+1}\Tilde\mbQ^{\prime}\right\|\right)\left\|\Tilde\mbQ-\Tilde\mbQ^{\prime}\right\|\\&+\underbrace{2\mu^{W_k} s\mathrm{Cvg}_{\theta}(k-1)+2s\sum_{i=1}^{k-1}\mu^{W_{k-i}}\mathrm{Cvg}_{\theta}(k-i-1)\Pi^{i}_{j=1}\left\|\delta\mbI-\mbW^{\top}_{\mathcal{S},k-j+1}\Tilde\mbQ^{\prime}\right\|}_{\text{Term}(I)}.
\end{aligned}
\end{equation}
Thus, the stability parameter for $\bar\mbQ$ is
\begin{equation}
\mathrm{Stab}^{\bar\mbQ}_{\theta}(k)
=\|\mbW_{\mathcal{S},k}\|\zeta_{\theta}(k-1)+\sum^{k-1}_{i=1}\|\mbW_{\mathcal{S},k-i}\|\zeta_{\theta}(k-i-1)\Pi^{i}_{j=1}\left\|\delta\mbI-\mbW^{\top}_{\mathcal{S},k-j+1}\Tilde\mbQ^{\prime}\right\|.
\end{equation}
To extend the stability parameter to all pairs in the dataset, we express it in terms of $\alpha_\theta$, as defined in Lemma~\ref{lem_1}:
\begin{equation}
\begin{aligned}
\mathrm{Stab}^{\bar\mbQ}_{\theta}(k) &= \mathcal{O}\left(\alpha^{k-1}_{\theta}+\mathrm{const}_{(1)}\alpha_{\theta}+\mathrm{const}_{(2)}+\sum^{k-1}_{i=1}\left(\alpha^{k-i-1}_{\theta}+\mathrm{const}_{(1)}\alpha_{\theta}+\mathrm{const}_{(2)}\right)\alpha^{i}_{\theta}\right),\\
&= \mathcal{O}\left(\alpha^{k-1}_{\theta}+(k-1)\alpha^{k-1}_{\theta}+\left(\alpha_{\theta}\mathrm{const}_{(1)}\sum^{k-1}_{i=0}\alpha^{i}_{\theta}\right)+\left(\mathrm{const}_{(2)}\sum^{k-1}_{i=0}\alpha^{i}_{\theta}\right)\right),\\
&= \mathcal{O}\left(k\alpha^{k-1}_{\theta}+\frac{\alpha_{\theta}-\alpha^{k+1}_{\theta}}{1-\alpha_{\theta}}+\frac{1-\alpha^{k}_{\theta}}{1-\alpha_{\theta}}\right),\\
&= \mathcal{O}\left(1-\alpha^{k}_{\theta}\right).
\end{aligned}
\end{equation}
The worst-case stability parameter is then defined as $\mathrm{Stab}^{\bar\mbQ}(k)=\sup_{\theta}\mathrm{Stab}^{\bar\mbQ}_{\theta}(k)$.
The stability parameter for $\bar\mbb$ is
\begin{equation}
\mathrm{Stab}^{\bar\mbb}_{\theta}(k)
=\|\mbW_{\mathcal{S},k}\|+\sum^{k-1}_{i=1}\|\mbW_{\mathcal{S},k-i}\|\Pi^{i}_{j=1}\left\|\delta\mbI-\mbW^{\top}_{\mathcal{S},k-j+1}\Tilde\mbQ^{\prime}\right\|,
\end{equation}
where for all datasets:
\begin{equation}
\begin{aligned}
\mathrm{Stab}^{\bar\mbb}_{\theta}(k)=\mathcal{O}\left(\sum^{k-1}_{i=0}\alpha^i_{\theta}\right)
=\mathcal{O}\left(1-\alpha^{k}_{\theta}\right). 
\end{aligned}   
\end{equation}
$\bullet$ \textbf{Note:} There exists a slight discrepancy between the bound presented in \eqref{stab2} and that in \eqref{stab_eq}. This difference arises solely from the existence of the term $\text{Term}(I)$. It is worth emphasizing that the effect of this term, expressed as $\sup_{\theta}\text{Term}(I)$, will explicitly appear in the generalization bound formulated in Theorem~\ref{theorem_1}.
\subsection{Hard-Thresholding}
\label{App_D2}
Similar to the stability analysis for the ST operator, by choosing $\theta_k$ as outlined in Appendix~\ref{App_C}, we can obtain the stability parameter for $\bar\mbQ$ as
\begin{equation}
\begin{aligned}
\mathrm{Stab}^{\bar\mbQ}_{\theta}(k)
=\|\mbW_{\mathcal{S},k}\|\zeta_{\theta}(k-1)+\sum^{k-1}_{i=1}\|\mbW_{\mathcal{S},k-i}\|\zeta_{\theta}(k-i-1)\Pi^{i}_{j=1}\left\|\delta\mbI-\mbW^{\top}_{\mathcal{S},k-j+1}\Tilde\mbQ^{\prime}\right\|,
\end{aligned}
\end{equation}
where the parameter $\zeta_{\theta}(k)$ is
\begin{equation}
\zeta_{\theta}(k)=\mathrm{Cvg}_{\theta}(k-1)\sup_{\Tilde\mbQ\in\mathcal{H}^{s\times s}_{r,L}}\left\|\delta\mbI-\mbW_{\mathcal{S},k}^{\mathrm{\top}}\Tilde\mbQ\right\|+\delta B,~\delta\in(0,1].
\end{equation}
Similarly, the stability parameter for $\bar\mbb$ is given by
\begin{equation}
\mathrm{Stab}^{\bar\mbb}_{\theta}(k)
=\|\mbW_{\mathcal{S},k}\|+\sum^{k-1}_{i=1}\|\mbW_{\mathcal{S},k-i}\|\Pi^{i}_{j=1}\left\|\delta\mbI-\mbW^{\top}_{\mathcal{S},k-j+1}\Tilde\mbQ^{\prime}\right\|.
\end{equation}
\section{Sensitivity Analysis For (Soft/Hard)-Thresholding Operators}
\label{App_E}
In this section, we analyze and derive the sensitivity parameters for both ST and HT update processes.
\subsection{Soft-Thresholding}
\label{App_E1}
By considering the algorithmic parameters $\{\mbW^{\prime}_i,\theta^{\prime}_i\}_{i=1}^{k}$, we define $\mbx^{\prime}_k$ as
\begin{equation}
\label{sens0}
\mbx^{\prime}_k=\operatorname{ST}_{\theta^{\prime}_k}\left(\delta\mbx^{\prime}_{k-1}-\mbW^{\prime\top}_{k}\left(\bar\mbQ\mbx^{\prime}_{k-1}-\bar\mbb\right)\right).
\end{equation}
To derive the sensitivity parameter for the ST update process, we impose the following two assumptions:
(i) for all $i \in [k]$ and for every $\mbQ$ in the dataset, both $\mbW_i$ and $\mbW_i^{\prime}$ belong to $\mathcal{ST}_{\mbW}(\mbQ;\delta)$; i.e.,
\begin{equation}
\begin{aligned}
&\sup_{\Tilde\mbQ\in\mathcal{H}^{s\times s}_{r,L}}\left\|\delta\mbI-\mbW_{\mathcal{S},i}^{\top}\Tilde\mbQ\right\|+\mu^{W_i} s\in(0,1),\\&\sup_{\Tilde\mbQ\in\mathcal{H}^{s\times s}_{r,L}}\left\|\delta\mbI-\mbW_{\mathcal{S},i}^{\prime\top}\Tilde\mbQ\right\|+\mu^{W^{\prime}_i} s\in(0,1).
\end{aligned}
\end{equation}
(ii) The threshold parameters $\{\theta_i\}$ are chosen according to \eqref{to7}, ensuring that each $\theta_i$ enables accurate support recovery of $\mbx_i$ over the set $\mathcal{S}$ for all $i\in[k]$. Similarly, the parameters $\theta_i^{\prime}$ are selected based on $\mbx^{\prime}_i$ to guarantee support recovery of $\mbx^{\prime}_i$ over the same support set $\mathcal{S}$ for all $i\in[k]$.
Define $\mathcal{F}_{k-1}$ as in \eqref{boz}, and let $\mathcal{F}_{k-1}^{\prime}$ be defined analogously.
Similarly, define $\mbz_{k-1}$ and $\mbz^{\prime}_{k-1}$ in the same way as $\mathcal{F}_{k-1}$ and $\mathcal{F}^{\prime}_{k-1}$, respectively, but without restricting them to the support set $\mathcal{S}$. We can then write
\begin{equation}
\label{sens1}
\begin{aligned}
\|\mbx_{k}-\mbx^{\prime}_k\|&= \left\|\operatorname{ST}_{\theta_k}(\mbz_{k-1})-\operatorname{ST}_{\theta_k^{\prime}}(\mbz^{\prime}_{k-1})\right\|,\\&=\left\|\mathcal{F}_{k-1}-\theta_k\operatorname{sign}\left(\mathcal{F}_{k-1}\right)-\mathcal{F}^{\prime}_{k-1}+\theta_k^{\prime}\operatorname{sign}\left(\mathcal{F}^{\prime}_{k-1}\right)\right\|,\\&=\left\|\left(\mathcal{F}_{k-1}-\mathcal{F}^{\prime}_{k-1}\right)+\left(\theta_k^{\prime}-\theta_k\right)\operatorname{sign}\left(\mathcal{F}^{\prime}_{k-1}\right)+\theta_k\left(\operatorname{sign}\left(\mathcal{F}^{\prime}_{k-1}\right)-\operatorname{sign}\left(\mathcal{F}_{k-1}\right)\right)\right\|,\\&\leq\underbrace{\left\|\mathcal{F}_{k-1}-\mathcal{F}^{\prime}_{k-1}\right\|}_{\text{Term}\star}+\underbrace{\left\|\left(\theta_k-\theta_k^{\prime}\right)\operatorname{sign}\left(\mathcal{F}^{\prime}_{k-1}\right)\right\|}_{\text{Term}\star\star}+\underbrace{\theta_k\left\|\operatorname{sign}\left(\mathcal{F}^{\prime}_{k-1}\right)-\operatorname{sign}\left(\mathcal{F}_{k-1}\right)\right\|}_{\text{Term}\star\star\star}.
\end{aligned}
\end{equation}
The second term, $\text{Term}\star\star$, can be upper bounded as
\begin{equation}
\label{ghoor_1}
\begin{aligned}
\text{Term}\star\star\leq\sqrt{s}\left|\theta_k-\theta_k^{\prime}\right|.
\end{aligned}
\end{equation}
Similarly, the third term, $\text{Term}\star\star\star$, admits the following bound:
\begin{equation}
\label{ghoor}
\begin{aligned}
\text{Term}\star\star\star&\leq2\sqrt{s}\theta_k,\\&\leq2\mu^{W_k}s\sup_{\mbx^{\mathrm{opt}}(\mbQ,\mbb)}\|\mbx_{k-1}-\mbx^{\mathrm{opt}}(\mbQ,\mbb)\|,\\&\leq2\mu^{W_k}s\mathrm{Cvg}_{\theta}(k-1).
\end{aligned}
\end{equation}
Also, $\text{Term}\star$ has the following bound
\begin{equation}
\label{ghoor_2}
\begin{aligned}
\text{Term}\star
&=\left\|\delta\mbx_{\mathcal{S},k-1}-\mbW_{\mathcal{S},k}^{\top}\left(\Tilde\mbQ\mbx_{\mathcal{S},k-1}-\Tilde\mbQ\mbx^{\mathrm{opt}}_{\mathcal{S}}\left(\mbQ,\mbb\right)\right)-\delta\mbx^{\prime}_{\mathcal{S},k-1}+\mbW_{\mathcal{S},k}^{\prime\top}\left(\Tilde\mbQ\mbx^{\prime}_{\mathcal{S},k-1}-\Tilde\mbQ\mbx^{\mathrm{opt}}_{\mathcal{S}}\left(\mbQ,\mbb\right)\right)\right\|,\\&\leq\left\|\delta\mbI-\mbW_{\mathcal{S},k}^{\top}\Tilde\mbQ\right\|\left\|\mbx_{k-1}-\mbx^{\prime}_{k-1}\right\|+\left\|\Tilde\mbQ\right\|\left\|\mbx^{\prime}_{k-1}-\mbx^{\mathrm{opt}}\left(\mbQ,\mbb\right)\right\|\left\|\mbW_{\mathcal{S},k}-\mbW_{\mathcal{S},k}^{\prime}\right\|,\\&\leq\left\|\delta\mbI-\mbW_{\mathcal{S},k}^{\top}\Tilde\mbQ\right\|\left\|\mbx_{k-1}-\mbx^{\prime}_{k-1}\right\|+\left\|\Tilde\mbQ\right\|\underbrace{\left\|\mbx^{\prime}_{k-1}-\mbx^{\mathrm{opt}}\left(\mbQ,\mbb\right)\right\|}_{\text{Term}\square}\left\|\mbW_{k}-\mbW_{k}^{\prime}\right\|.
\end{aligned}
\end{equation}
In \eqref{ghoor_2}, $\text{Term}\square$ can be bounded by $\mathrm{Cvg}_{\theta^{\prime}}(k-1)$. 
By combining \eqref{ghoor_1}, \eqref{ghoor}, and \eqref{ghoor_2}, and under the assumption that $\mbx_0=\mbx^{\prime}_0$, we can expand \eqref{sens1} for all $k$ layers and write
\begin{equation}
\label{sens3}
\begin{aligned}
\left\|\mbx_{k}-\mbx^{\prime}_k\right\|&\leq \left\|\Tilde\mbQ\right\|\mathrm{Cvg}_{\theta^{\prime}}(k-1)
\left\|\mbW_{k}-\mbW^{\prime}_{k}\right\|+\sqrt{s}|\theta_k-\theta^{\prime}_k|+2\mu^{W_k} s\mathrm{Cvg}_{\theta}(k-1)\\&+\sum_{i=1}^{k-1}\left\|\mbW_{k-i}-\mbW^{\prime}_{k-i}\right\|\left\|\Tilde\mbQ\right\|\mathrm{Cvg}_{\theta^{\prime}}(k-i-1)\Pi^{i}_{j=1}\left\|\delta\mbI-\mbW^{\top}_{\mathcal{S},k-j+1}\Tilde\mbQ\right\|\\&+\sum_{i=1}^{k-1}\left|\theta_{k-i}-\theta^{\prime}_{k-i}\right|\sqrt{s}\Pi^{i}_{j=1}\left\|\delta\mbI-\mbW^{\top}_{\mathcal{S},k-j+1}\Tilde\mbQ\right\|\\&+2s\sum_{i=1}^{k-1}\mu^{W_{k-i}}\mathrm{Cvg}_{\theta}(k-i-1)\Pi^{i}_{j=1}\left\|\delta\mbI-\mbW^{\top}_{\mathcal{S},k-j+1}\Tilde\mbQ\right\|.
\end{aligned}
\end{equation}
The bound in \eqref{sens3} holds only for a single data pair $\left(\mbQ,\mbb\right)$. To extend this result to the entire dataset, we write
\begin{equation}
\label{sens3_new}
\begin{aligned}
\left\|\mbx_{k}-\mbx^{\prime}_k\right\|&\leq \left[\sup_{\Tilde\mbQ\in\mathcal{H}^{s\times s}_{r,L}}\left\|\Tilde\mbQ\right\|\right]\mathrm{Cvg}_{\theta^{\prime}}(k-1)
\left\|\mbW_{k}-\mbW^{\prime}_{k}\right\|+\sqrt{s}|\theta_k-\theta^{\prime}_k|+\underbrace{2\mu^{W_k} s\mathrm{Cvg}_{\theta}(k-1)}_{\text{Term}(a)}\\&+\sum_{i=1}^{k-1}\left\|\mbW_{k-i}-\mbW^{\prime}_{k-i}\right\|\left[\sup_{\Tilde\mbQ\in\mathcal{H}^{s\times s}_{r,L}}\left\|\Tilde\mbQ\right\|\right]\mathrm{Cvg}_{\theta^{\prime}}(k-i-1)\Pi^{i}_{j=1}\left[\sup_{\Tilde\mbQ\in\mathcal{H}^{s\times s}_{r,L}}\left\|\delta\mbI-\mbW^{\top}_{\mathcal{S},k-j+1}\Tilde\mbQ\right\|\right]\\&+\sum_{i=1}^{k-1}\left|\theta_{k-i}-\theta^{\prime}_{k-i}\right|\sqrt{s}\Pi^{i}_{j=1}\left[\sup_{\Tilde\mbQ\in\mathcal{H}^{s\times s}_{r,L}}\left\|\delta\mbI-\mbW^{\top}_{\mathcal{S},k-j+1}\Tilde\mbQ\right\|\right]\\&+\underbrace{2s\sum_{i=1}^{k-1}\mu^{W_{k-i}}\mathrm{Cvg}_{\theta}(k-i-1)\Pi^{i}_{j=1}\left[\sup_{\Tilde\mbQ\in\mathcal{H}^{s\times s}_{r,L}}\left\|\delta\mbI-\mbW^{\top}_{\mathcal{S},k-j+1}\Tilde\mbQ\right\|\right]}_{\text{Term}(b)}.
\end{aligned}
\end{equation}
Without loss of generality, assume that $\mathrm{Cvg}_{\theta}(i)\geq\mathrm{Cvg}_{\theta^{\prime}}(i)$ for all $i\in[k]$.
The corresponding sensitivity parameter associated with $\{\mbW_i,\theta_i\}_{i=1}^{k}$ can then be expressed as
\begin{equation}
\label{sens4}
\begin{aligned}
\mathrm{Sens}^{\mbW_k}_\theta(k)&= \left[\sup_{\Tilde\mbQ\in\mathcal{H}^{s\times s}_{r,L}}\left\|\Tilde\mbQ\right\|\right]\mathrm{Cvg}_{\theta}(k-1),\\\mathrm{Sens}^{\mbW_{k-i}}_\theta(k)&=\left[\sup_{\Tilde\mbQ\in\mathcal{H}^{s\times s}_{r,L}}\left\|\Tilde\mbQ\right\|\right]\mathrm{Cvg}_{\theta}(k-i-1)\Pi^{i}_{j=1}\left[\sup_{\Tilde\mbQ\in\mathcal{H}^{s\times s}_{r,L}}\left\|\delta\mbI-\mbW^{\top}_{\mathcal{S},k-j+1}\Tilde\mbQ\right\|\right],~i\in[k-1],\\\mathrm{Sens}^{\theta_k}_\theta(k)&=\sqrt{s},\\\mathrm{Sens}^{\theta_{k-i}}_\theta(k)&=\sqrt{s}\Pi^{i}_{j=1}\left[\sup_{\Tilde\mbQ\in\mathcal{H}^{s\times s}_{r,L}}\left\|\delta\mbI-\mbW^{\top}_{\mathcal{S},k-j+1}\Tilde\mbQ\right\|\right],~i\in[k-1].
\end{aligned}
\end{equation}
$\bullet$ \textbf{Note:} There exists a slight discrepancy between the bound presented in \eqref{sens3_new} and that in \eqref{sens_eq}. This difference arises solely from the existence of the terms $\text{Term}(a)$ and $\text{Term}(b)$. It is worth emphasizing that the effect of these two terms, expressed as $\sup_{\theta}\left[\text{Term}(a) + \text{Term}(b)\right]$, will explicitly appear in the generalization bound formulated in Theorem~\ref{theorem_1}.
\subsection{Hard-Thresholding}
\label{App_E2}
In the sensitivity analysis of the HT operator, the threshold parameter $\theta_k$ does not appear in the resulting upper bound.
Following a similar reasoning to that used for the ST operator, the sensitivity parameter corresponding to $\{\mbW_i\}_{i=1}^{k}$ can be written as
\begin{equation}
\begin{aligned}
\mathrm{Sens}^{\mbW_k}_\theta(k)&= \left[\sup_{\Tilde\mbQ\in\mathcal{H}^{s\times s}_{r,L}}\left\|\Tilde\mbQ\right\|\right]\mathrm{Cvg}_{\theta}(k-1),\\\mathrm{Sens}^{\mbW_{k-i}}_\theta(k)&=\left[\sup_{\Tilde\mbQ\in\mathcal{H}^{s\times s}_{r,L}}\left\|\Tilde\mbQ\right\|\right]\mathrm{Cvg}_{\theta}(k-i-1)\Pi^{i}_{j=1}\left[\sup_{\Tilde\mbQ\in\mathcal{H}^{s\times s}_{r,L}}\left\|\delta\mbI-\mbW^{\top}_{\mathcal{S},k-j+1}\Tilde\mbQ\right\|\right],~i\in[k-1].
\end{aligned}
\end{equation}
\section{Further Numerical Analysis}
\label{App_F}
In this section, we present a comprehensive numerical analysis and ablation study of the proposed scheme and its key parameters. We begin by introducing the bit count formula used to report the number of bits. Next, we demonstrate how one-bit quantization can smooth out inconsistencies in the algorithm’s layer-wise behavior, which is ideally expected to be monotonic. We then analyze the computational time required for each stage of the algorithm to assess its practical usability. Finally, we conduct an in-depth investigation of two critical parameters: the scale quantization parameter $\lambda$ and the modification parameter $\delta$, which is introduced to support the theoretical guarantees. We evaluate the impact of these parameters and discuss their selection strategies across both synthetic and real datasets.
\subsection{Bit Count Analysis}
\label{App_F1}
The calculation of the number of bits required for the models in Table~\ref{table_1} is summarized in Table~\ref{table_numbits}.

\subsection{Convergence Inconsistencies For DUN}
\label{App_F2}
In Fig.~\ref{figure_2}, we compare the convergence of DUN with high-resolution weights against one-bit weights, presenting both training and test results. One expectation from algorithm unrolling is that it should mimic the monotonic decreasing behavior of classical optimization solvers across iterations, corresponding to layers in the unrolled network. However, as shown in Fig.~\ref{figure_2}, the DUN deviates from this trend in some layer instances, exhibiting significant inconsistencies, a phenomenon thoroughly analyzed in \cite{heaton2023safeguarded}. In our empirical findings, we observe that binarizing the linear weights can improve the convergence of the unrolled algorithm in the sense of reducing these inconsistencies, as illustrated in Fig.~\ref{figure_2}(b)-(c). In particular, the convergence inconsistency in DUN results in a significantly high NMSE of $4$ dB in the $21$st layer. In contrast, the one-bit DUN maintains a much lower error, with NMSEs of $-2.5$ dB using Lazy projection and $-5$ dB with the $\ell_1$ regularization method. These findings highlight the effectiveness of binarization in reducing convergence inconsistencies. In high-resolution DUNs, the variability in weight magnitudes and directions across iterations can
lead to inconsistent convergence behaviors, including potential increases in the objective function.
Binarization standardizes the weight magnitudes, leading to more uniform influence of each weight on the updates.
Moreover, the discrete nature of binarized weights reduces the variance in the updates. Since the
weights can only take on two values, the updates become more predictable, and the optimization
process is less susceptible to the erratic behaviors caused by fluctuating weight values. Additionally, Fig.~\ref{figure_2}(b) and (c) compare one-bit DUN using lazy projection versus $\ell_1$-regularization. As shown, the latter achieves superior performance.
\begin{figure*}[t]
	\centering
	\subfloat[]
		{\includegraphics[width=0.3\columnwidth]{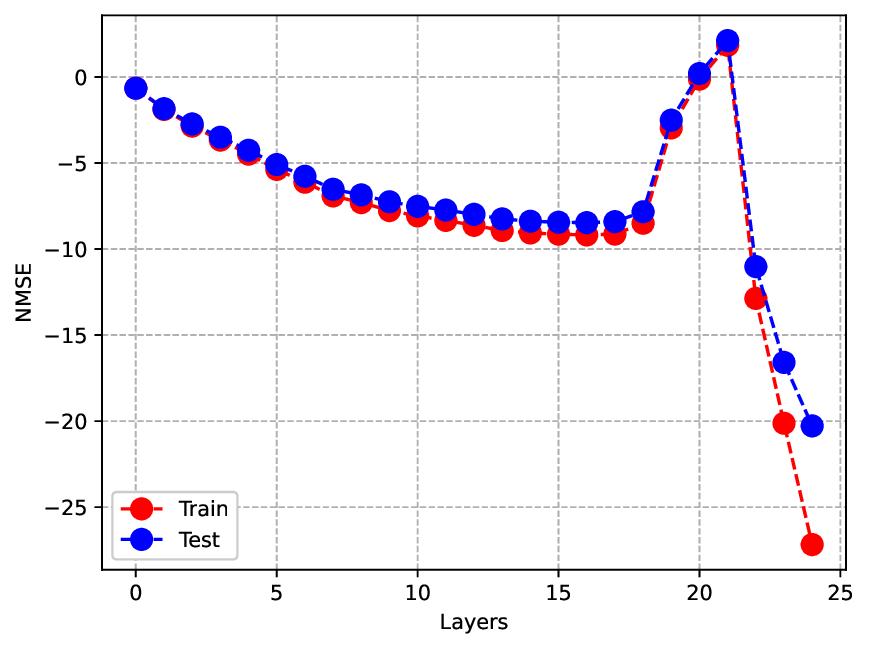}}\quad
    \subfloat[]
		{\includegraphics[width=0.3\columnwidth]{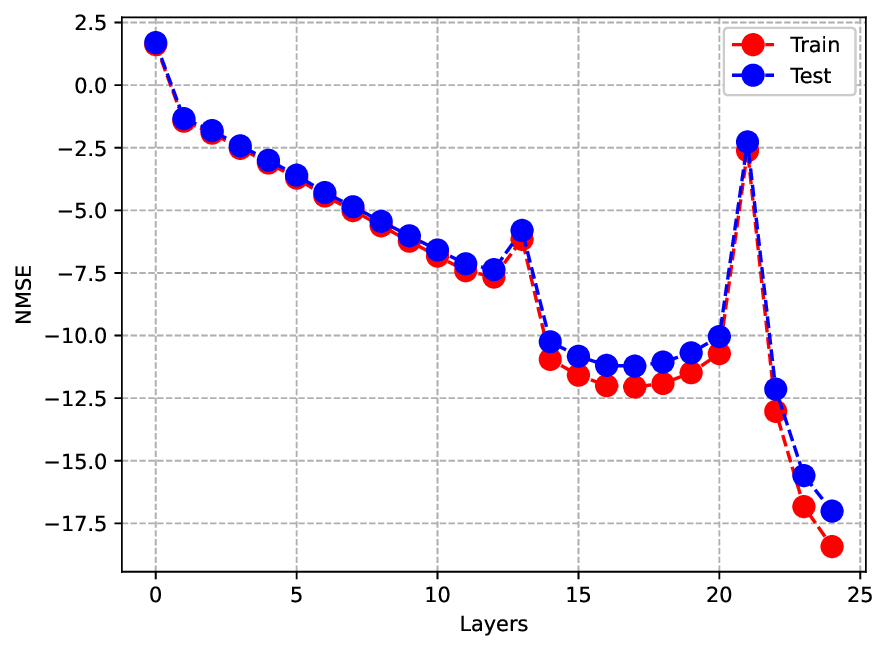}}\quad
    \subfloat[]
		{\includegraphics[width=0.3\columnwidth]{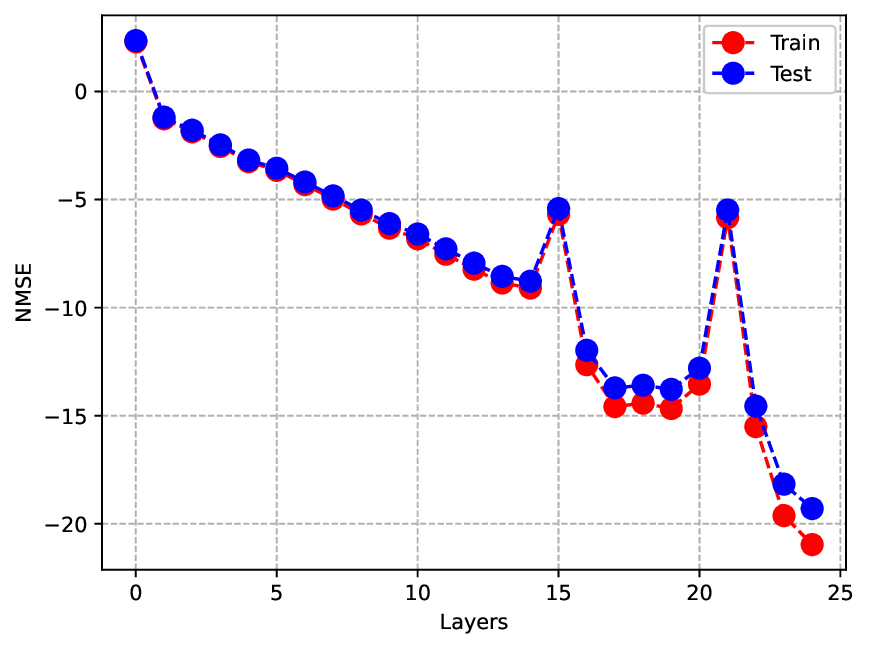}}
	\caption{Comparison of per-layer error across both training and test stages for three models: (a) DUN with high-resolution weights, (b) one-bit DUN utilizing lazy projection, and (c) one-bit DUN employing $\ell_1$-regularization.}
\label{figure_2}
\end{figure*}
\begin{table}[t]
\caption{Summary of the storage bit formulations for the FCN with ReLU, FCN with ST, DUN, and One-Bit DUN models, as presented in Table~\ref{table_1}.}
\centering
\setlength{\tabcolsep}{4pt}
\begin{tabular}{  c || c }
\hline
 &  $\#$ Bits \\[0.5 ex]
\hline
FCN with ReLU & $32(mn+Kn^2)$ \\
FCN with ST & $32(mn+Kn^2)+32K$ \\
DUN & $32K(mn+1)$ \\
One-Bit DUN & $K(mn+32)$ \\
\hline
\end{tabular}
\label{table_numbits}
\end{table}
\begin{table}[ht]
\caption{The CPU time $(\text{min}.\text{sec})$ of different stages of the proposed one-bit algorithm unrolling method.}
\centering
\setlength{\tabcolsep}{4pt}
\begin{tabular}{  c || c | c | c | c| c | c}
\hline
$\#$ Layers &  $5$ & $10$ & $15$& $20$& $22$& $25$ \\[0.5 ex]
\hline
Pretraining & $1.21$ & $1.59$& $2.45$  & $3.33$& $3.51$& $4.00$ \\
QAT & $1.17$ & $2.12$& $3.00$  & $3.31$& $3.48$  & $3.59$\\
Scale update & $0.30$ & $0.34$& $0.32$  & $0.35$ & $0.34$  & $0.33$\\
\hline
\end{tabular}
\label{table_5}
\end{table}
\begin{table}[ht]
\caption{Training/Test NMSE comparison between learned and fixed scale values, with and without \textbf{Stage~II} of the proposed one-bit unrolling method.}
\centering
\setlength{\tabcolsep}{4pt}
\begin{tabular}{  c || c | c | c | c}
\hline
$\lambda_0\lambda$ &  $0.0604$ & $0.02$ & $0.05$& $0.1$ \\[0.5 ex]
\hline
 Training NMSE (dB)& $-20.96$ & $-15.01$& $-15.94$  & $-8.03$ \\
 Test NMSE (dB)& $-19.30$ & $-13.20$& $-13.83$  & $-6.01$\\
\hline
\end{tabular}
\label{table_60}
\end{table}
\begin{table}[ht]
\caption{Scale values obtained after \textbf{Stage~II} of the proposed one-bit algorithm unrolling method, shown as a function of the number of layers.}
\centering
\setlength{\tabcolsep}{4pt}
\begin{tabular}{  c || c | c | c | c | c | c }
\hline
$\#$ Layers &  $5$ & $10$ & $15$& $20$  & $22$ & $25$ \\[0.5 ex]
\hline
$\lambda_0\lambda$ & $0.0501$ & $0.0594$& $0.0592$  & $0.0618$ & $0.0734$$25$ & $0.0604$\\
\hline
\end{tabular}
\label{table_20}
\end{table}
\subsection{Computational Complexity Analysis}
\label{App_F3}
Although QAT enables more effective quantization by adapting to the model structure, it typically incurs higher computational cost due to the required pretraining. However, in our case, this overhead is minimal since the DUN architecture has significantly fewer parameters than standard FCNs and requires less training data to achieve strong performance. As shown in Table~\ref{table_5}, the QAT stage takes roughly the same time as pretraining, and the scale parameter update is fast—taking only a few seconds—as it involves learning a single value. The experimental setting for these results is identical to that used in Table~\ref{table_1}. All experiments are conducted using 2 vCPUs and 1 GPU. We report the average NMSE over 15 random seeds, using the ADAM optimizer for training.

\subsection{Ablation Study of \texorpdfstring{$\lambda$}{lambda}}
\label{App_F40}
To demonstrate the impact of the data-driven design of the quantization scale parameter, we report performance results in Table~\ref{table_60}, using the same experimental setting as in Table~\ref{table_1}. As observed, the learned scale parameter significantly improves final performance. Notably, the best result was obtained with a scale value of $0.0604$, which emerged from the scale update process in Stage~II of our proposed one-bit unrolling algorithm, starting from an initial value of $\lambda_0 = 0.02$. In contrast, the other values remained fixed and were not updated during this process.

Additionally, Table~\ref{table_20} reports the learned scale parameter for each DUN configuration with varying layer counts, corresponding to the performance results presented in Table~\ref{table_1}.
\begin{figure*}[t]
	\centering
	\subfloat[$K=5,\delta=1$.]
		{\includegraphics[width=0.2\columnwidth]{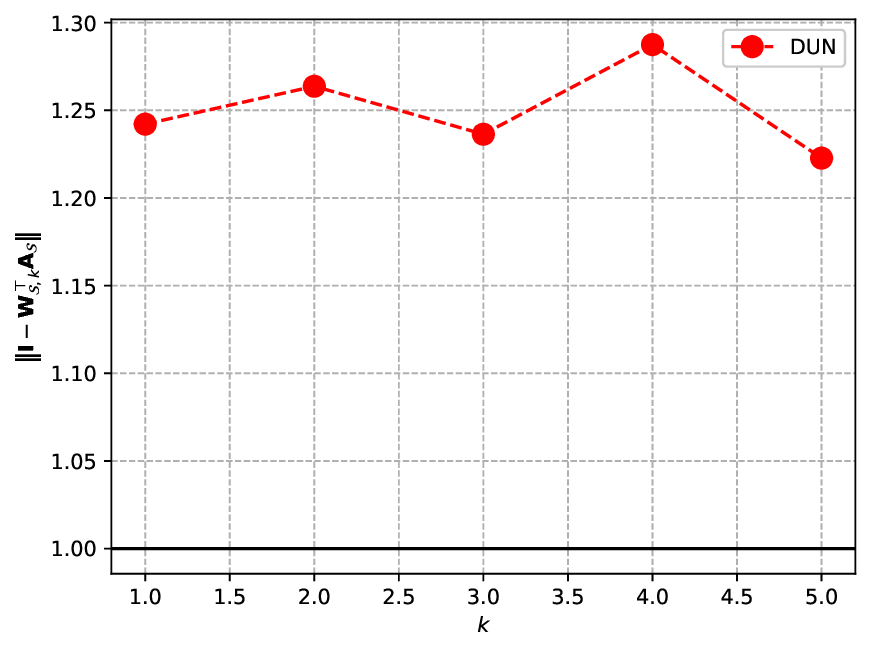}}\quad
    \subfloat[$K=5,\delta=1,\lambda=0.02$.]
		{\includegraphics[width=0.2\columnwidth]{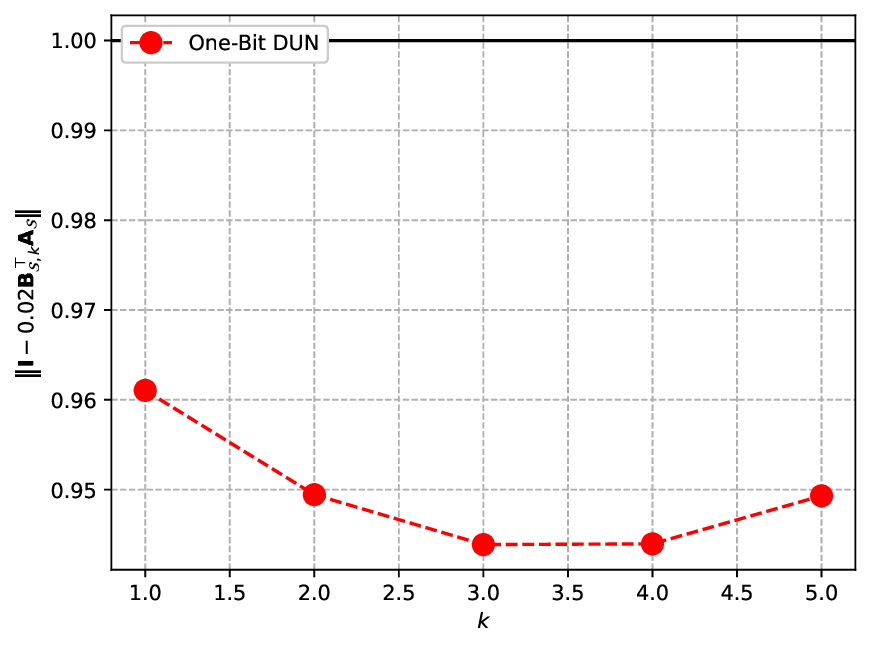}}\quad
    \subfloat[$K=10,\delta=1$.]
		{\includegraphics[width=0.2\columnwidth]{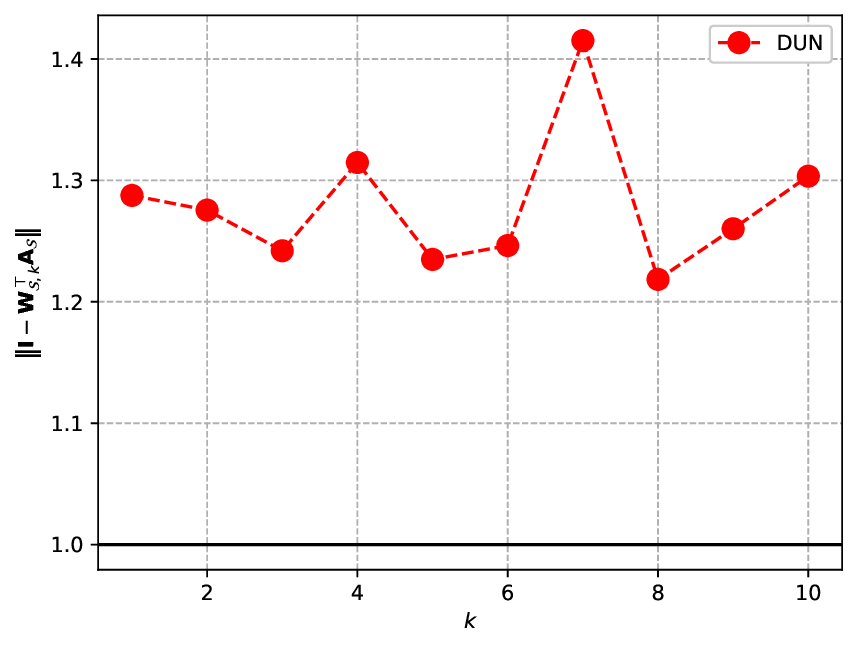}}\quad
    \subfloat[$K=10,\delta=1,\lambda=0.02$.]
		{\includegraphics[width=0.2\columnwidth]{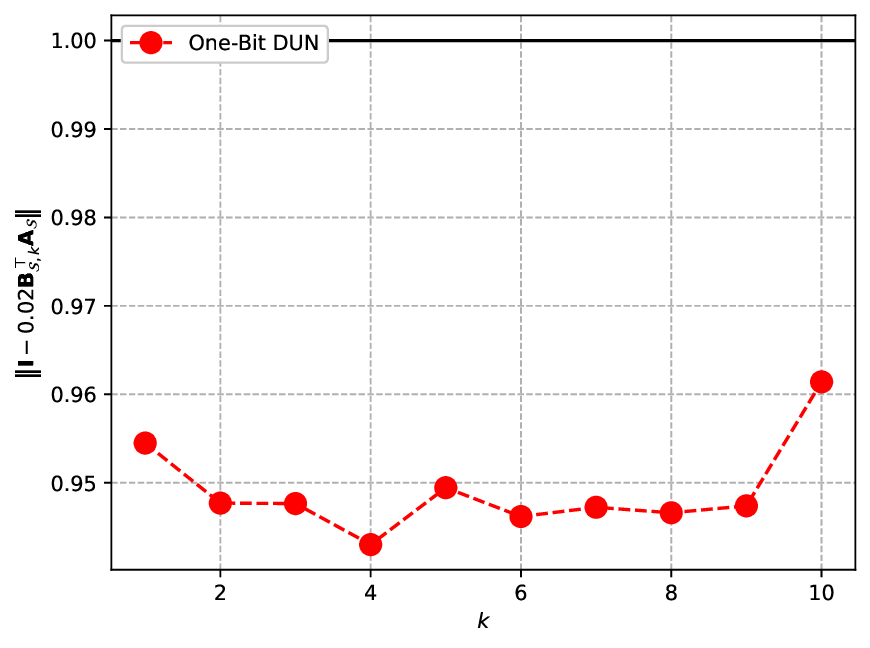}}
	\caption{Spectral plots of the HT update process for (a) DUN with $K=5,\delta=1$, (b) one-bit DUN with $K=5,\delta=1,\lambda=0.02$, (c) DUN with $K=10,\delta=1$, and (d) one-bit DUN with $K=10,\delta=1,\lambda=0.02$. These results illustrate that the training process naturally learns binary weights from the set defined in \eqref{ht1}.}
\label{figure_3}
\end{figure*}
\subsection{Ablation Study of \texorpdfstring{$\delta$}{delta}}
\label{App_F4}

In this section, we present ablation experiments to systematically examine the impact of $\delta$ on convergence and generalization bounds. 
Specifically, we use the same data settings as in Fig.~\ref{figure_1}(a), but with a fixed sparse set $\mathcal{S}$ of size $|\mathcal{S}|=10$ for the HT operator and $|\mathcal{S}|=5$ for the ST operator, where the non-zero indices of $\mbx^{\mathrm{opt}}$ are drawn from $\mathcal{S}$. 
The learning rates for \textbf{Stage~I} and \textbf{Stage~II} of our proposed one-bit algorithm unrolling remain consistent with previous settings.

\textbf{Hard-Thresholding:} As established in Theorem~\ref{theorem_1}, the generalization ability of the HT update process is determined by three key factors: convergence, stability, and sensitivity. In Appendices~\ref{App_C}, \ref{App_D2}, and \ref{App_E2}, we demonstrated that the key parameters governing these characteristics for DUN and one-bit DUN are given by $f_k\triangleq\left\|\delta\mbI-\mbW_{\mathcal{S},k}^{\top}\mbA_{\mathcal{S}}\right\|$ for DUN and $f_k\triangleq\left\|\delta\mbI-\lambda\mbB_{\mathcal{S},k}^{\top}\mbA_{\mathcal{S}}\right\|$ for one-bit DUN, respectively\footnote{The guarantees are established for the HT update process using the Gram matrix $\mbQ$. However, the theoretical results can also be extended to the HT update process with the fixed sensing matrix $\mbA$.}. Fig.~\ref{figure_3}(a) and (b) present the values of $f_k$ for $k\in[5]$ in DUN and one-bit DUN, respectively. Similarly, Fig.~\ref{figure_3}(c) and (d) depict the corresponding values for $k\in[10]$. In these experiments, we set $\delta=1$, aligning with the classical update process proposed in \cite{chen2018theoretical} with a slight modification incorporating the HT operator. As observed, DUN (without binarization) inherently does not learn the weight matrices in the set \eqref{ht1}. In contrast, the one-bit DUN, with an appropriately chosen scale value $\lambda$, naturally learns the binary weight matrices belonging to \eqref{ht1}. This finding is remarkable, as it suggests that the HT update process with binary weights inherently learns these structures without the need for any additional parameters, such as $\delta$. As extensively discussed in Appendices~\ref{App_C}, \ref{App_D2}, and \ref{App_E2}, the condition $f_k<1$ for all $k\in[K]$ ensures a bounded generalization gap. This condition is confirmed in Fig.~\ref{figure_3}(b) and (d). The corresponding results for train/test NMSE and the generalization gap are reported in Table~\ref{table_7}(left panel) and Table~\ref{table_8}(right panel) for $K=5$ and $K=10$, respectively.

In the second part of our analysis, we examine the impact of $\delta$ on convergence and generalization ability. As previously discussed, setting $\delta=1$ enables the one-bit DUN with the HT update process to naturally learn binary weight matrices conforming to \eqref{ht1}. Under this condition, Appendix~\ref{App_C} confirms that the convergence of the update process is guaranteed.
In contrast, when $\delta\neq 1$, convergence generally deteriorates due to the emergence of an additional term in the convergence upper bound, as characterized in \eqref{ht5}. Nonetheless, smaller values of $\delta$ lead to reduced spectral terms $f_k$ during training. According to Theorem~\ref{theorem_1}, this introduces a trade-off in the generalization ability of the HT update process: while a smaller $\delta$ improves the spectral terms that govern stability and sensitivity, it simultaneously impairs the convergence parameter that also appears in their upper bounds.
We examined this phenomenon in Table~\ref{table_7}(left panel) and Table~\ref{table_8}(right panel), which report results for the one-bit DUN with $K=5$ and $K=10$, respectively. The corresponding spectral values $f_k$ are visualized in Fig.~\ref{figure_5}, based on the results presented in these tables. As evident from the figure, reducing the value of $\delta$ can improve the spectral terms $f_k$ (by decreasing their magnitude); however, this does not necessarily translate to better generalization performance, as it may simultaneously hinder convergence.

\begin{figure*}[t]
	\centering
	\subfloat[$\delta=1$]
		{\includegraphics[width=0.22\columnwidth]{spectral_plot_binary_5_delta_1_lambda_0.02.eps}}\quad
    \subfloat[$\delta=0.95$]
		{\includegraphics[width=0.22\columnwidth]{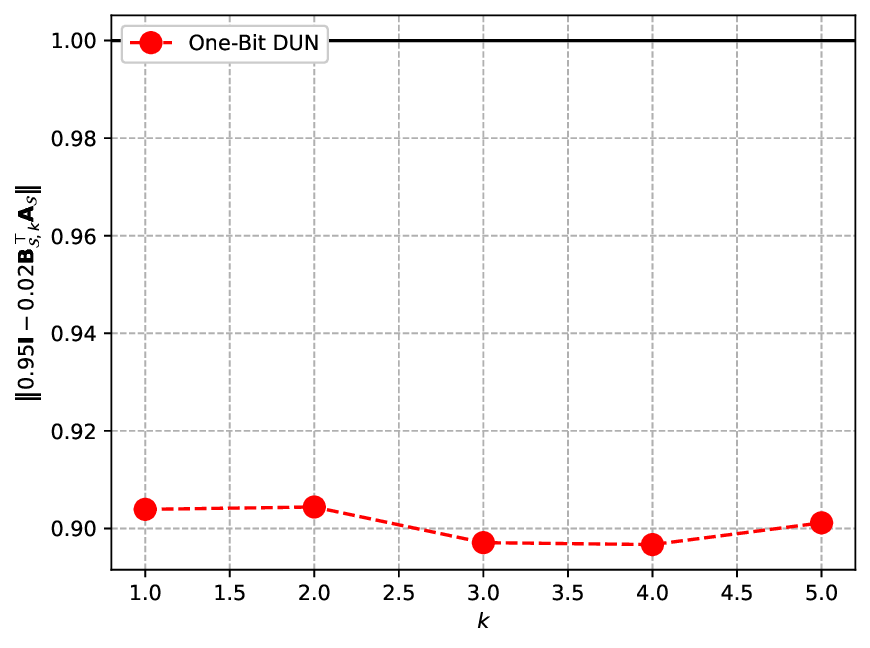}}
    \subfloat[$\delta=0.9$]
		{\includegraphics[width=0.22\columnwidth]{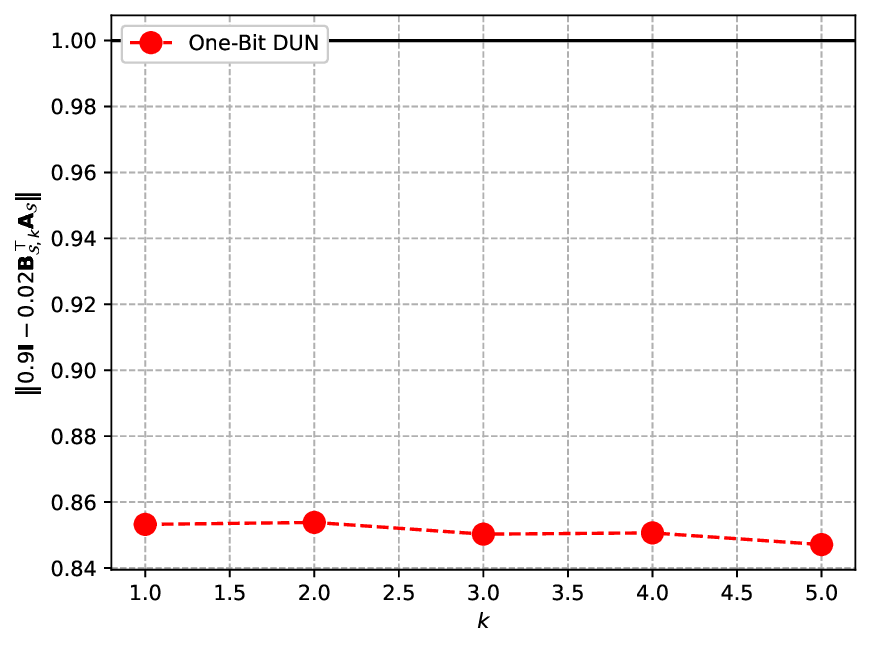}}\quad
    \subfloat[$\delta=0.85$]
		{\includegraphics[width=0.22\columnwidth]{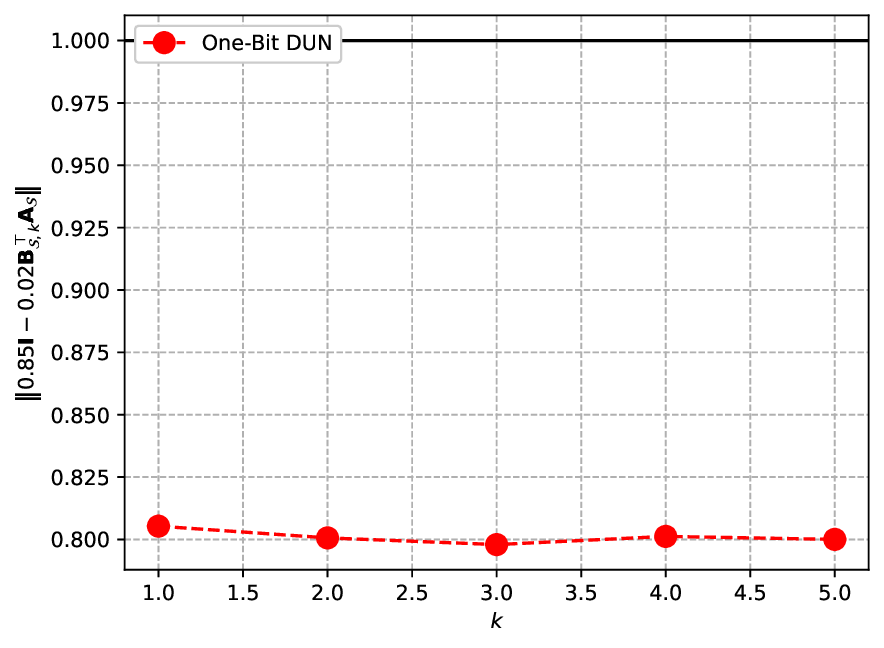}}\\
    \subfloat[$\delta=1$]
		{\includegraphics[width=0.22\columnwidth]{spectral_plot_binary_10_delta_1_lambda_0.02.eps}}\quad
    \subfloat[$\delta=0.95$]
		{\includegraphics[width=0.22\columnwidth]{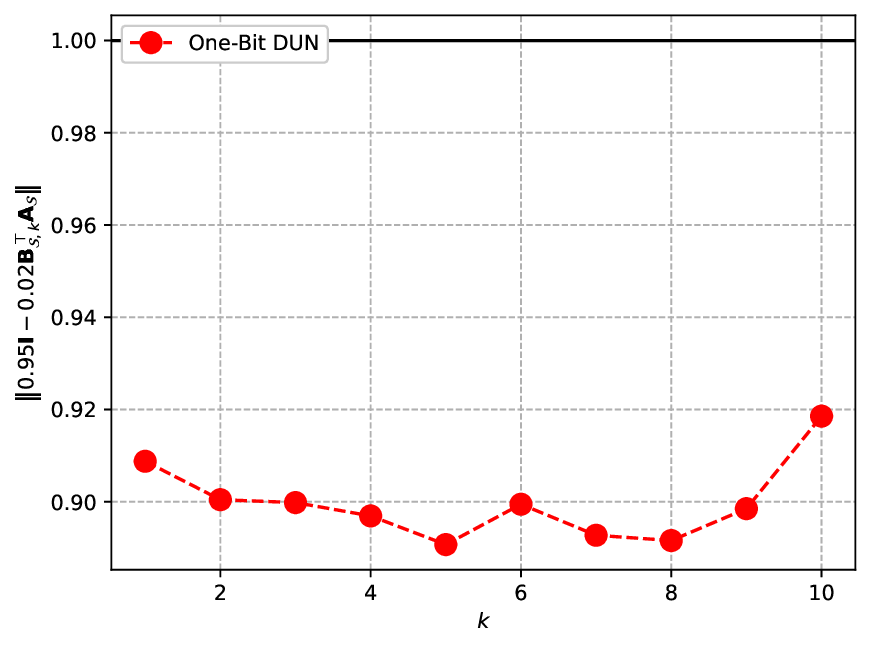}}
    \subfloat[$\delta=0.9$]
		{\includegraphics[width=0.22\columnwidth]{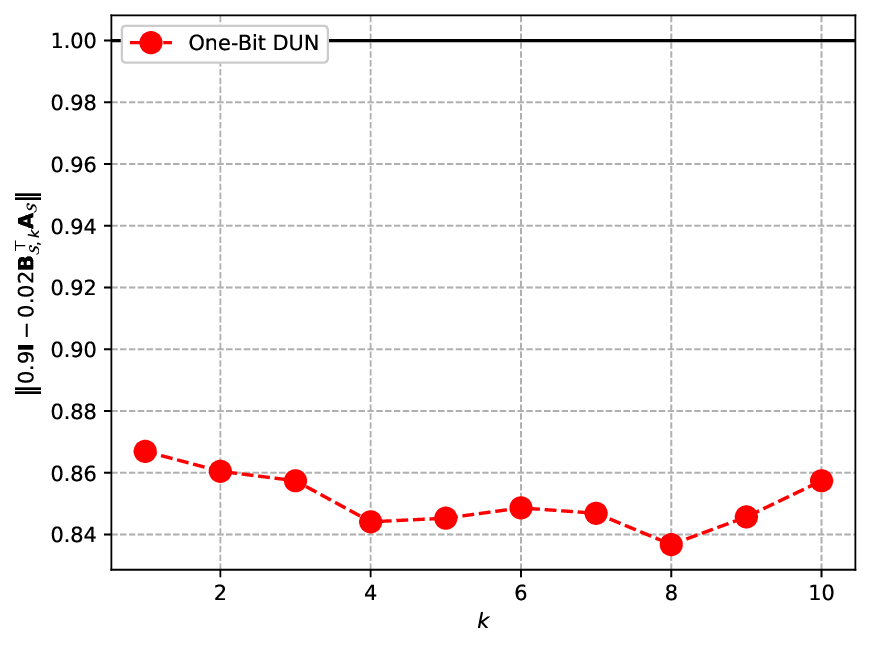}}\quad
    \subfloat[$\delta=0.85$]
		{\includegraphics[width=0.22\columnwidth]{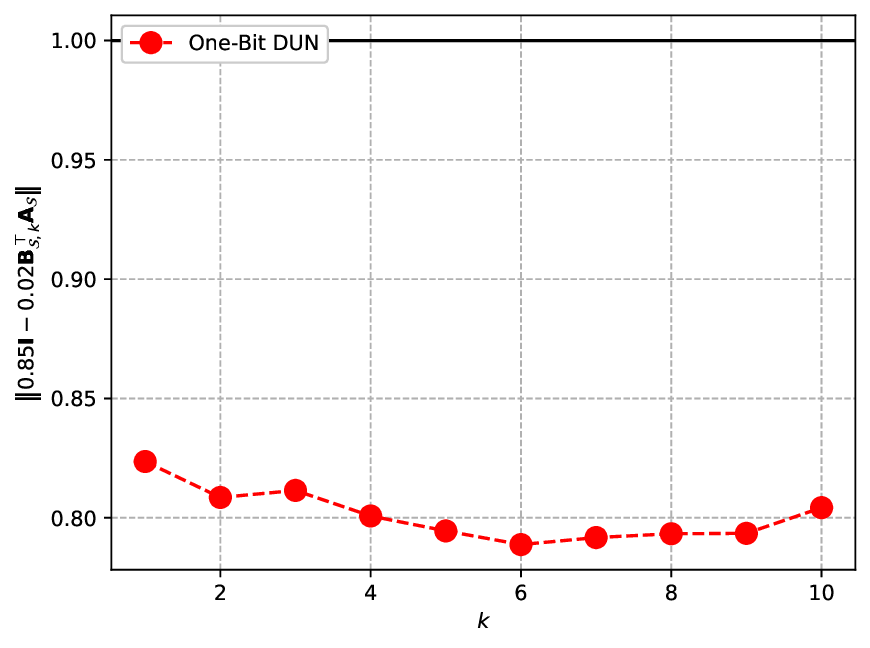}}
	\caption{Spectral plots of the HT update process for the one-bit DUN with 
    (a) $K=5,\delta=1$, (b) $K=5,\delta=0.95$, (c) $K=5,\delta=0.9$, (d) $K=5,\delta=0.85$, (e) $K=10,\delta=1$, (f) $K=10,\delta=0.95$, (g) $K=10,\delta=0.9$, and (h) $K=10,\delta=0.85$. All results are obtained with a fixed scale value of $\lambda=0.02$.}
\label{figure_5}
\end{figure*}
\begin{figure}[t]
	\centering
    \subfloat[DUN with $K=10,\delta=1$.]
		{\includegraphics[width=0.3\columnwidth]{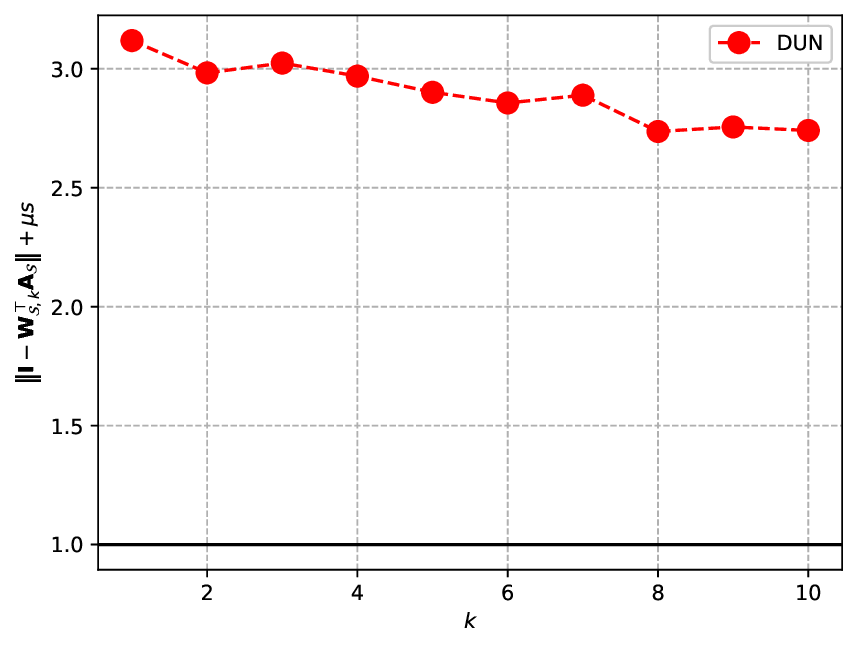}}\quad
	\subfloat[One-Bit DUN with $K=10,\delta=1,\lambda=0.015$.]
		{\includegraphics[width=0.3\columnwidth]{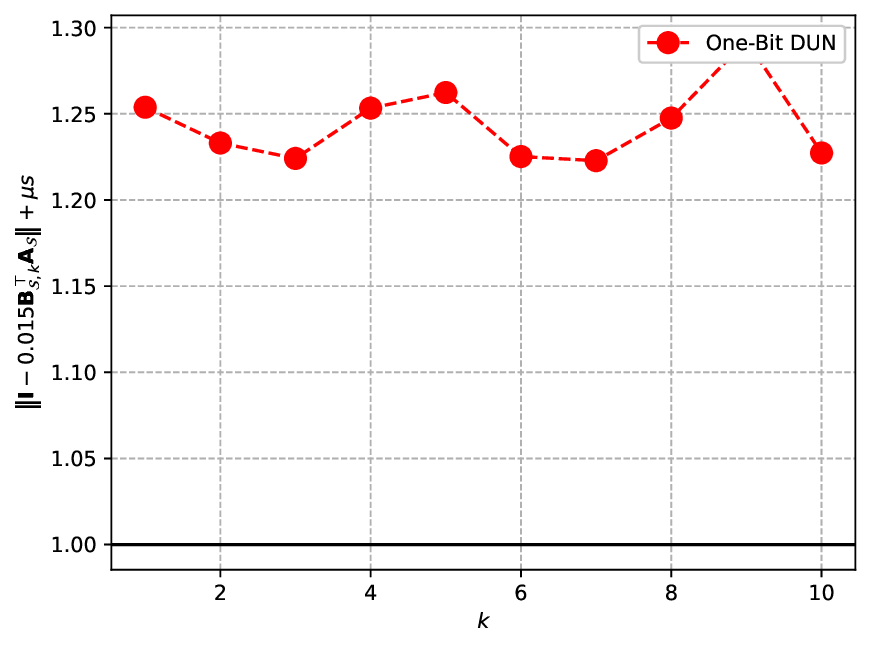}}\quad
    \subfloat[One-Bit DUN with $K=10,\delta=0.3,\lambda=0.015$.]
		{\includegraphics[width=0.3\columnwidth]{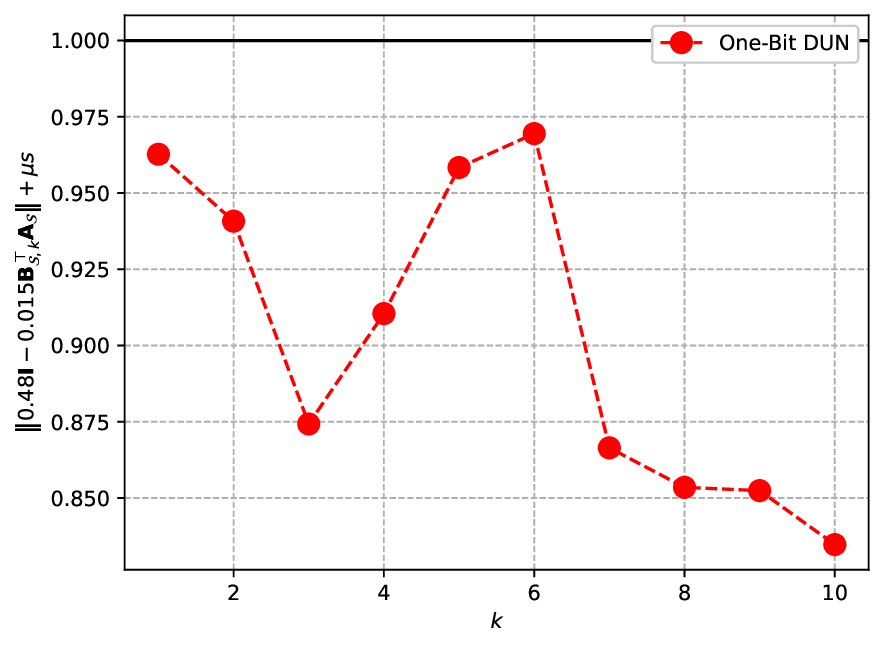}}
	\caption{Spectral plots of the ST update process for (a) DUN with $K=10,\delta=1$, (b) one-bit DUN with $K=10,\delta=1,\lambda=0.015$, and (c) one-bit DUN with $K=10,\delta=0.48,\lambda=0.015$. The results indicate that neither DUN nor one-bit DUN with $\delta=1$ learns weight matrices from the set defined in \eqref{to10}. However, by selecting an appropriate $\delta$ (e.g., $\delta=0.48$), the one-bit DUN is able to learn weights that belong to the set \eqref{to10}.}
\label{figure_4}
\end{figure}

\begin{figure}[t]
	\centering
    \subfloat[One-Bit DUN with $K=5,\delta=1,\lambda=0.02$.]
		{\includegraphics[width=0.3\columnwidth]{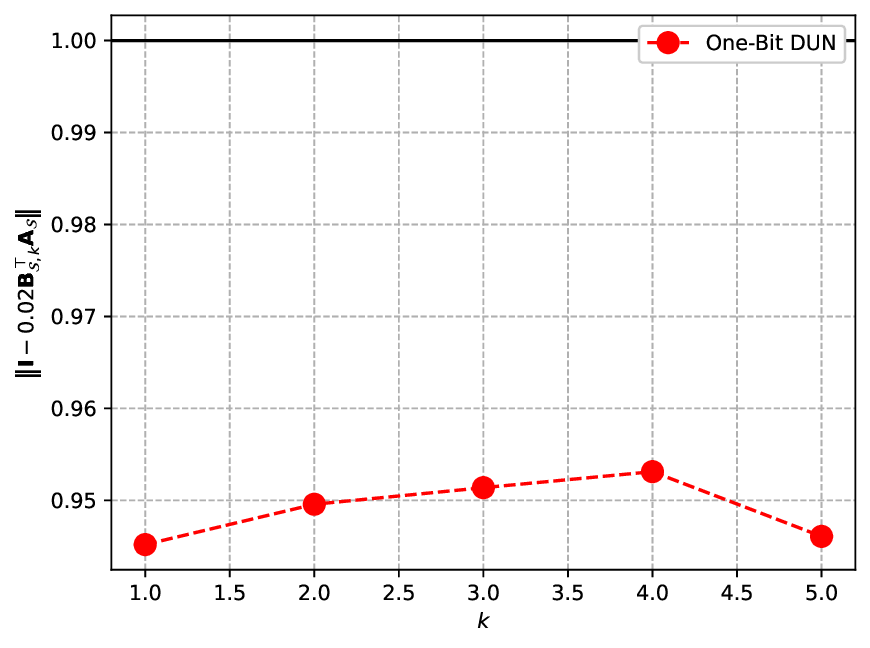}}\quad
	\subfloat[One-Bit DUN with $K=10,\delta=1,\lambda=0.02$.]
		{\includegraphics[width=0.3\columnwidth]{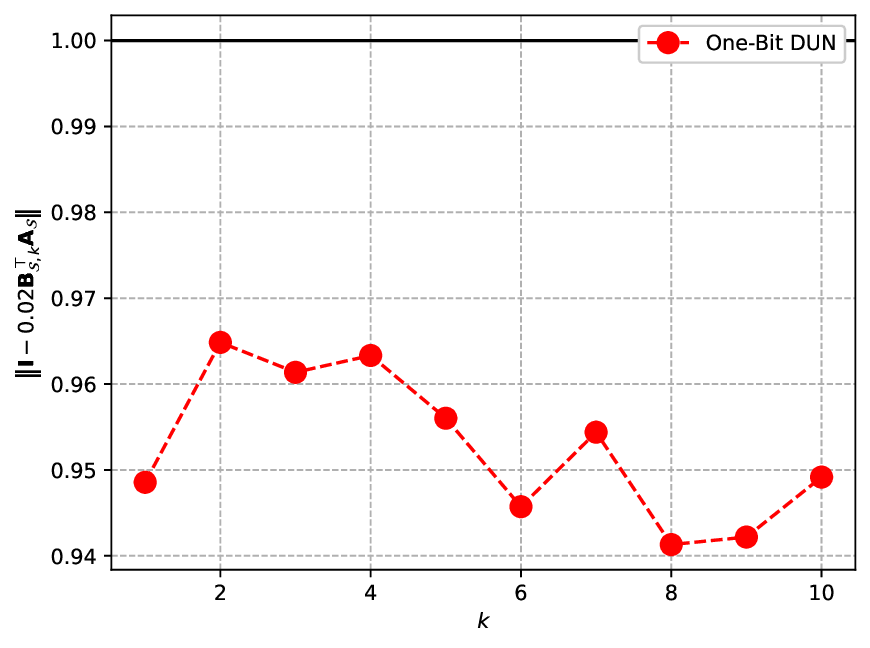}}
	\caption{Spectral plots of the form $f_k=\left\|\delta\mbI-\lambda\mbB_{\mathcal{S},k}^{\top}\mbA_{\mathcal{S}}\right\|$ for the ST update process in the one-bit DUN, shown for (a) $K=5,\delta=1$, (b) $K=10,\delta=1$. All results are obtained with a fixed scale value of $\lambda=0.02$.}
\label{figure_st_new}
\end{figure}
\begin{table*}[t]
\centering
\caption{Training NMSE, test NMSE, and generalization gap for one-bit DUN using the HT update process with (left) $K=5$ and (right) $K=10$ layers.}
\begin{minipage}{0.48\linewidth}
\centering
\resizebox{\linewidth}{!}{
\begin{tabular}{  c || c | c | c }
\hline
$\delta$ &  Training NMSE (dB) & Test NMSE (dB) & Generalization Gap (dB) \\[0.5 ex]
\hline
$1$ & $-6.27$ & $-6.27$& $0.0044$ \\
$0.95$ & $-6.27$ & $-6.27$& $0.0062$ \\
$0.9$ & $-6.06$ & $-6.06$& $0.0017$ \\
$0.85$ & $-5.79$ & $-5.78$& $0.0086$ \\
\hline
\end{tabular}
}
\label{table_7}
\end{minipage}
\hfill
\begin{minipage}{0.48\linewidth}
\centering
\resizebox{\linewidth}{!}{
\begin{tabular}{  c || c | c | c }
\hline
$\delta$ &  Training NMSE (dB) & Test NMSE (dB) & Generalization Gap (dB) \\[0.5 ex]
\hline
$1$ & $-8.47$ & $-8.45$& $0.0121$ \\
$0.95$ & $-7.51$ & $-7.49$& $0.0187$ \\
$0.9$ & $-6.73$ & $-6.73$& $0.0033$ \\
$0.85$ & $-6.57$ & $-6.55$& $0.0217$ \\
\hline
\end{tabular}
}
\label{table_8}
\end{minipage}
\end{table*}
\begin{table}[ht]
\caption{Training NMSE, test NMSE, and generalization gap for one-bit DUN using the ST update process with $K=10$ layers.}
\centering
\setlength{\tabcolsep}{4pt}
\begin{tabular}{  c || c | c | c }
\hline
$\delta$ &  Training NMSE (dB) & Test NMSE (dB) & Generalization Gap (dB) \\[0.5 ex]
\hline
$1$ & $-8.49$ & $-8.46$ & $0.03$ \\
$0.48$ & $-5.48$ & $-5.47$ & $0.01$ \\
\hline
\end{tabular}
\label{table_9}
\end{table}
\begin{table}[ht]
\caption{Training NMSE, test NMSE, and generalization gap for one-bit DUN using the ST update process with $K=5$ and $K=10$ layers.}
\centering
\setlength{\tabcolsep}{4pt}
\begin{tabular}{  c || c | c | c }
\hline
$\#$ Layers &  Training NMSE (dB) & Test NMSE (dB) & Generalization Gap (dB) \\[0.5 ex]
\hline
$5$ & $-8.30$ & $-8.29$ & $0.0143$ \\
$10$ & $-9.19$ & $-9.16$ & $0.0296$ \\
\hline
\end{tabular}
\label{table_st_new}
\end{table}
\begin{table}[ht]
\caption{Training NMSE, test NMSE, and generalization bound for high-resolution DUN using the ST update process with $K=20$ on the BSD500 dataset.}
\centering
\setlength{\tabcolsep}{4pt}
\begin{tabular}{  c || c | c | c }
\hline
$\delta$ &  Training NMSE (dB) & Test NMSE (dB) & Generalization Gap (dB) \\[0.5 ex]
\hline
$1$ & $-17.15$ & $-16.70$ & $0.45$ \\
$0.9$ & $-17.10$ & $-17.06$ & $0.04$ \\
$0.8$ & $-17.40$ & $-17.08$ & $0.32$ \\
\hline
\end{tabular}
\label{table_10}
\end{table}
\begin{table*}[ht]
\caption{Training NMSE, test NMSE, and generalization bound for one-bit DUN using the ST update process with $K=20$ on the BSD500 dataset.}
\centering
\setlength{\tabcolsep}{4pt}
\begin{tabular}{  c || c | c | c }
\hline
$\delta$ &  Training NMSE (dB) & Test NMSE (dB) & Generalization Gap (dB) \\[0.5 ex]
\hline
$1$ & $-15.99$ & $-15.68$ & $0.31$ \\
$0.9$ & $-14.16$ & $-14.15$ & $0.01$ \\
$0.8$ & $-10.51$ & $-10.50$ & $0.01$ \\
\hline
\end{tabular}
\label{table_11}
\end{table*}
\begin{table}[ht]
\caption{Training NMSE, test NMSE, and generalization bound for high-resolution DUN using the HT update process with $K=20$ on the BSD500 dataset.}
\centering
\setlength{\tabcolsep}{4pt}
\begin{tabular}{  c || c | c | c }
\hline
$\delta$ &  Training NMSE (dB) & Test NMSE (dB) & Generalization Gap (dB) \\[0.5 ex]
\hline
$1$ & $-14.63$ & $-14.28$ & $0.35$ \\
$0.9$ & $-15.52$ & $-15.15$ & $0.37$ \\
$0.8$ & $-15.55$ & $-15.17$ & $0.38$ \\
\hline
\end{tabular}
\label{table_12}
\end{table}
\begin{table}[ht]
\caption{Training NMSE, test NMSE, and generalization bound for one-bit DUN using the HT update process with $K=20$ on the BSD500 dataset.}
\centering
\setlength{\tabcolsep}{4pt}
\begin{tabular}{  c || c | c | c }
\hline
$\delta$ &  Training NMSE (dB) & Test NMSE (dB) & Generalization Gap (dB) \\[0.5 ex]
\hline
$1$ & $-12.61$ & $-12.60$ & $0.01$ \\
$0.9$ & $-11.14$ & $-10.96$ & $0.18$ \\
$0.8$ & $-9.04$ & $-9.01$ & $0.03$ \\
\hline
\end{tabular}
\label{table_13}
\end{table}
\textbf{Soft-Thresholding:} 
In Appendices~\ref{App_B}, \ref{App_D1}, and \ref{App_E2}, we established that the key parameters governing the convergence, stability, and sensitivity of the ST update process are given by $f_k\triangleq\left\|\delta\mbI-\mbW^{\top}_{\mathcal{S},k}\mbA_{\mathcal{S}}\right\|+\mu s$ for DUN and $f_k\triangleq\left\|\delta\mbI-\lambda\mbB^{\top}_{\mathcal{S},k}\mbA_{\mathcal{S}}\right\|+\mu s$ for one-bit DUN, respectively. Fig.~\ref{figure_4}(a) and (b) show these spectral terms for $k \in [10]$ with $\delta=1$ for the ST update process in DUN and one-bit DUN, respectively. Unlike the HT update process, neither DUN nor one-bit DUN learns weights belonging to the set \eqref{to10}. However, it is notable that the spectral terms $f_k$ are consistently smaller for the one-bit DUN. Table~\ref{table_9} presents the corresponding results for train/test NMSE and the generalization gap. It is important to emphasize that Theorem~\ref{theorem_1} provides a sufficient condition for ensuring a bounded generalization error. Specifically, if a learning algorithm demonstrates bounded convergence, stability, and sensitivity, then Theorem~\ref{theorem_1} guarantees that it will also exhibit bounded generalization ability. This is not the case for $\delta=1$ according to Appendices~\ref{App_B}, \ref{App_D1}, and \ref{App_E2} since the spectral terms $f_k$ are not bounded by one.

In the next experiment, we examine the effect of the parameter $\delta$ on the convergence and generalization ability of the ST update process. Similar to our argument for the HT update process, decreasing $\delta$ generally leads to degraded convergence performance. This trend is supported both theoretically by the convergence upper bound in \eqref{to6}, and empirically by the results presented in Table~\ref{table_9}. However, reducing $\delta$ also leads to smaller spectral terms $f_k$, as illustrated in Fig.~\ref{figure_4}(c). In particular, for $\delta=0.48$, all spectral terms $f_k$ are bounded by $1$. According to Theorem~\ref{theorem_1}, this reduction in spectral terms could improve the generalization ability of the ST update process. This is because the spectral terms directly influence the upper bounds on stability and sensitivity, as analyzed in Appendices~\ref{App_D1} and \ref{App_E2}, respectively.

Building on our earlier findings with the HT update process where the one-bit DUN learns binary weights from the set \eqref{ht1} when $\delta=1$, we now examine the spectral terms of the form $f_k=\left\|\delta\mbI-\lambda\mbB_{\mathcal{S},k}^{\top}\mbA_{\mathcal{S}}\right\|$ for the one-bit DUN using the ST update process. As shown in Fig.~\ref{figure_st_new}(a) and (b), it is noteworthy that the one-bit DUN with the ST update process also learns binary weights consistent with the set \eqref{ht1} when $\delta=1$. The corresponding results for train/test NMSE and the generalization gap are reported in Table~\ref{table_st_new} for both $K=5$ and $K=10$ layers. Note that in this experiment, we set $\left|\mathcal{S}\right| = 10$ and the scale parameter $\lambda=0.2$.

Herein, we conduct another numerical ablation study to examine the effect of $\delta$ on the performance and generalization ability of the one-bit DUN using the BSD500 dataset. Specifically, we compare both high-resolution and one-bit DUNs under two proximal operators—ST and HT—whose theoretical convergence guarantees were previously discussed. The experimental setting is identical to that used in the results reported in Table~\ref{table_14} for a $50\%$ CS ratio. As shown in Table~\ref{table_10}, introducing $\delta = 0.9$ slightly degrades training performance but notably improves generalization. With $\delta = 0.8$, the generalization gap is larger than that of $\delta = 0.9$, yet still better than the case without $\delta$. Interestingly, for the high-resolution network, the performance degradation from adding $\delta$ is negligible.

For the one-bit DUN with the ST operator, whose results are reported in Table~\ref{table_11}, applying $\delta = 0.9$ reduces the generalization gap to $0.01$, marking a significant improvement despite a slight drop in training performance. With $\delta = 0.8$, the degradation in both training and test accuracy becomes more noticeable, yet it yields the lowest generalization gap among all settings.

For the case where the high-resolution DUN employs the HT operator, Table~\ref{table_12} shows that introducing $\delta = 0.9$ or $\delta = 0.8$ surprisingly improves both training and test performance. However, this enhancement comes at the cost of a larger generalization gap. It is important to note that the upper bound on the generalization gap provided in Theorem~\ref{theorem_1} is statistical and reflects a worst-case scenario. As this experiment demonstrates, reducing the spectral norm $\alpha_{\phi}$ via $\delta$ does not always lead to improved generalization. Nonetheless, in all cases, adding $\delta$ does not degrade performance and, in some instances, can even improve it.

For the one-bit DUN with the HT operator, whose results are reported in Table~\ref{table_13}, setting $\delta = 0.9$ leads to degradation in both generalization gap and performance across training and test stages, indicating that this choice is suboptimal for the experiment. While $\delta = 0.8$ yields a smaller generalization gap than $\delta = 0.9$, it still performs worse than the $\delta = 1$ case in terms of both generalization and overall accuracy.

Based on this experiment using the BSD500 dataset, we observe that introducing the parameter $\delta$ generally enhances the generalization performance of high-resolution DUNs. However, in the case of one-bit DUNs, which naturally drive the spectral norm $\alpha_{\phi}$ below $1$ even without the use of $\delta$, quantization alone often suffices to enhance generalization. In fact, introducing $\delta$ in the one-bit setting can lead to degradation in both training and test performance. While the theoretical guarantees provided in this work are statistical and reflect worst-case scenarios, our empirical results offer a practical insight: incorporating $\delta$ is beneficial for high-resolution DUNs, but is typically unnecessary for one-bit DUNs.

\section{Overlap Analysis of Sparsity Patterns}
\label{sss}
Beyond overall sparsification ratios, it is also important to assess how well different schemes align with the problem-driven sparse structure. We compute the overlap between the zero positions introduced by ternary quantization and those predefined by the physics of the problem. The results show an overlap of $48\%$ on the large dataset and $63\%$ on the BSD500 dataset. These values indicate that while both methods induce sparsity, a significant portion of the zeros selected by ternary do not coincide with the problem’s inherent structure. For the large model, ternary zeros overlap with PIBiNN’s physics-driven zeros by only $48\%$, indicating that ternary often removes different (and potentially essential) couplings. In contrast, for BSD500 the overlap is $63\%$, suggesting ternary partially aligns with the physics structure in this setting. These results highlight that PIBiNN achieves systematic, problem-consistent sparsity, whereas ternary’s sparsity pattern is dataset-dependent and less predictable.

\section{Limitations and Future Work}
\label{I}
A key limitation of our current study lies in latency. While the proposed physics-inspired binarization achieves strong compression and accuracy trade-offs, our implementation relies on CPU-based reference kernels for packed inference. As a result, the latency results do not yet reflect the full efficiency potential of the approach. We emphasize, however, that this is primarily a systems-level concern: the core architectural and theoretical contributions of this work remain valid, and practical speedups can be realized by integrating optimized GPU or hardware kernels, which we leave to future systems research.

More broadly, our study is intentionally scoped to inverse problems with well-defined physics, where problem-driven sparsity provides meaningful structure. Extending the framework to more general domains will require additional validation. Likewise, our choice of a single data-driven scale per model emphasizes conceptual clarity and minimal parameter overhead. Future work may explore mixed-precision strategies or hybrid quantization schemes (e.g., combining one-scale and channel-wise quantization depending on layer importance), to balance performance and flexibility.

Overall, while our focus has been on demonstrating architectural and theoretical advantages, complementary efforts in systems optimization and broader application domains present natural next steps.

\end{document}

%% file: symbols.tex
\usepackage{amsmath}
\usepackage{algpseudocode}
\usepackage{amsthm}
\usepackage{dsfont}
\usepackage{algorithm}
\usepackage{thmtools}
\usepackage{amssymb}

\def\btheta{\boldsymbol{\theta}}

\def\bxi{\boldsymbol{\xi}}

\def\bPhi{\boldsymbol{\Phi}}

\def\mbb{\mathbf{b}}

\def\mbn{\mathbf{n}}

\def\mbx{\mathbf{x}}
\def\mby{\mathbf{y}}
\def\mbz{\mathbf{z}}

\def\mbA{\mathbf{A}}
\def\mbB{\mathbf{B}}

\def\mbD{\mathbf{D}}

\def\mbI{\mathbf{I}}

\def\mbQ{\mathbf{Q}}
\def\mbR{\mathbf{R}}

\def\mbW{\mathbf{W}}
\def\mbX{\mathbf{X}}
\def\mbY{\mathbf{Y}}

\newtheorem{theorem}{Theorem}

\newtheorem{lemma}{Lemma}

\theoremstyle{definition}
\newtheorem{definition}{Definition}

\algnewcommand\algorithmicinput{\textbf{Input:}}
\algnewcommand\Input{\item[\algorithmicinput]}
\algnewcommand\algorithmicoutput{\textbf{Output:}}
\algnewcommand\Output{\item[\algorithmicoutput]}
\algnewcommand\algorithmicinit{\textbf{Initialize:}}
\algnewcommand\Init{\item[\algorithmicinit]}
